\documentclass[11pt,reqno,twoside]{article}
\usepackage{mysty}
\usepackage{wrapfig}
\usepackage{fullpage}
\usepackage{calc}
\usepackage{patchcmd}
\makeatletter
\patchcommand\@starttoc{\begin{quote}}{\end{quote}}
\makeatother

\usepackage{enumitem}

\SetLabelAlign{parright}{\parbox[t]{\labelwidth}{\raggedleft{#1}}}
\setlist[description]{style=multiline,topsep=4pt,align=parright}

\makeatletter
\let\reftagform@=\tagform@
\def\tagform@#1{\maketag@@@{(\ignorespaces\textcolor{black}{#1}\unskip\@@italiccorr)}}
\newcommand{\iref}[1]{\textup{\reftagform@{\tcr{\ref{#1}}}}}
\makeatother

\newcommand{\kurdyka}{Kurdyka-\L ojasiewicz }

\newcommand{\exR}{\bar{\R}}
\newcommand{\jac}{\mathrm{Jac}}
\newcommand{\cljac}{\mathrm{Jac}^c}
\newcommand{\consjac}{\mathcal{J}}
\newcommand{\diff}{\mathrm{diff}}

\newcommand{\A}{\mathcal{A}}
\newcommand{\B}{\mathcal{B}}
\newcommand{\xstar}{x^{\star}}

\newcommand{\ystar}{y^{\star}}
\newcommand{\fb}{H}

\newcommand{\res}{\mbox{\rm res}\,}

\newcommand{\RR}{\mathbb{R}}

\begin{document}
%%%%%%%%%%%%%%%%%%%%%%%%%%%%%%%%%%%
\title{Differentiating Nonsmooth Solutions to Parametric Monotone Inclusion Problems}
\author{J\'{e}r\^{o}me Bolte\thanks{Toulouse School of Economics, University of Toulouse} \and
Edouard Pauwels\thanks{IRIT, CNRS, Universit\'e Toulouse III Paul Sabatier. Institut Universitaire de France (IUF)}\and 
Antonio Silveti-Falls\thanks{CVN, CentraleSup\'{e}lec, Universit\'{e} Paris-Saclay}
}
\date{}
\maketitle
\begin{flushleft}\end{flushleft}
%%%%%%%%%%%%%%%%%%%%%%%%%%%%%%%%%%%
\begin{abstract}
%Understanding the differentiability and regularity of the solution to a monotone inclusion problem is an important question with consequences for convex optimization, machine learning, signal processing, and beyond. Past attempts have been made either under very restrictive assumptions that ensure the solution is continuously differentiable or using mathematical tools that are incompatible with automatic differentiation. In this paper, 
We leverage path differentiability and a recent result on nonsmooth implicit differentiation calculus to give sufficient conditions ensuring that the solution to a monotone inclusion problem will be path differentiable, with formulas for computing its generalized gradient. A direct consequence of our result is that these solutions happen to be differentiable almost everywhere.  Our approach is fully compatible with automatic differentiation and comes with assumptions which are easy to check, roughly speaking: semialgebraicity and strong monotonicity. We illustrate the scope of our results by considering three fundamental composite problem settings: strongly convex problems, dual solutions to convex minimization problems and primal-dual solutions to min-max problems.
\end{abstract}

\begin{keywords}
Maximal monotone operator, monotone inclusion, generalized equation, implicit differentiation, differentiating solutions, Clarke subdifferential, generalized gradient, conservative field.
\end{keywords}

\begin{AMS}
49J40, 49J52, 49J53, 49K40, 65K15, 68T07
\end{AMS}

\section{Introduction}

Consider the following parametric maximal monotone inclusion problem
\nnewq{\label{eq:introProblem}
0\in\A_\theta(x) + \B_\theta(x)
}
where $\theta$ is some parameter and, for each $\theta$, $\A_\theta\colon\R^n\rightrightarrows\R^n$ is maximal monotone (possibly set-valued) and $\B_\theta\colon\R^n\to\R^n$ is maximal monotone and Lipschitz continuous. For a fixed $\theta$, a problem of this form is called a \emph{generalized equation} \cite{robinson1979generalized} or \emph{variational inequality}, and it models a wide range of optimization problems \cite{beck2017first}. In fact, designing algorithms to find solution to maximal monotone inclusions is at the heart of convex optimization \cite{bauschke2011convex, combettes2004solving, combettes2005signal, boct2015inertial, boct2016inertial, palaniappan2016stochastic, malitsky2020forward} and consequently has received a lot of attention in the last decades \cite{auslender2000lagrangian, bauschke2004finding, borwein2010fifty, attouch2019modern}.

Assuming that the solution $\xstar(\theta)$ is unique for each $\theta$,  our main objective in this paper is to investigate the regularity and differentiability of $\xstar$ with respect to $\theta$ as well as to develop calculus rules for computing a generalized derivative associated to it. In general, the solution $\xstar$ is not a function of $\theta$ because there can be several solutions for a given $\theta$ and, even under the assumption that $\xstar$ is unique, it is not guaranteed that $\xstar$ will be differentiable with respect to $\theta$, therefore motivating further study.
%The generalized equation \eqref{eq:introProblem} encompasses many convex optimization problems as special cases. Taking $A_\theta\equiv 0$, we can recover the setting for studying gradient descent, and taking $\B_\theta\equiv 0$ we can recover the setting for studying the proximal point algorithm \cite{rockafellar1976monotone}. In fact, finding solutions to maximal monotone inclusions is at the heart of convex optimization \cite{bauschke2011convex} and consequently has received a lot of attention in the last decades. The main thrust of research in this domain has been focused on creating and studying the convergence properties of splitting algorithms that iteratively solve problems like \eqref{eq:introProblem}. In contrast, in this paper, we do not develop algorithms to solve maximal monotone inclusions; instead, we study the differentiability properties of the solutions of maximal monotone inclusions with respect to some parameters.
Understanding the regularity of $\xstar$ as a function of $\theta$ has important applications in several areas: in deep learning for neural networks with implicit layers defined through monotone inclusions (e.g., monotone operator deep equilibrium networks \cite{winston2020monotone}, OptNet \cite{amos2017optnet}), in machine learning (hyperparameter tuning \cite{bolte2021nonsmooth}, meta-learning \cite{finn2017model}, dataset distillation \cite{blondel2021efficient}, adversarial examples \cite{liu2021investigating}), signal processing (\cite{combettes2004solving, combettes2005signal, fadili2009monotone}), and general nonsmooth bilevel optimization \cite{jane2005necessary, wang2022perturbation}, without being exhaustive. For this reason, many other works have studied regularity properties of solutions to various forms of \eqref{eq:introProblem}, using a myriad of different techniques that we will now discuss.

    \paragraph{Variational analytic methods} The study of solution mappings to parametric generalized equations can be traced back to variational analysis in the 1980s \cite{robinson1979generalized, robinson1983generalized}; \cite{robinson1980strongly} examined the Lipschitz continuity of the solution mapping when the single-valued monotone operator in the generalized equation is parametrized. This continued with results in \cite{king1992sensitivity, levy1994sensitivity, shapiro2003sensitivity} showing further the stability of the solution mapping but again avoiding parametrizing the nonsmooth/set-valued monotone operator. Another variational analytic approach is to use the notion of protodifferentiability developed in \cite{rockafellar1989proto}, which is used in \cite{adly2021sensitivity} to analyze the stability of solutions to parametric monotone inclusions. This approach was extended in \cite{wachsmuth2022resolvents} to show that the directional differentiability of the solution map under the assumption that the Lipschitz continuous operator is strongly monotone. Both of these works consider generalized equations and allow for both the Lipschitz and set-valued monotone operators to be parametrized, the same as the current work. In the language of sensitivity analysis, all data in the problem \eqref{eq:introProblem} can be perturbed. A similar approach to \cite{adly2021sensitivity} is used in \cite{berk2022lasso} to analyze the Lipschitz constant of the solution mapping to the lasso problem as a function of the penalization parameter. Finally, we mention that similar methods have also been applied to study the differentiability of the $\prox$ operator in \cite{poliquin1996generalized}, which is a special case of a generalized equation in which the Lipschitz operator is identically~$0$.

\paragraph{Nonsmooth implicit differentiation} As was already discussed, implicit differentiation is an important tool for characterizing properties of the solution to monotone inclusion problems, yet its application to nonsmooth problems remains a challenge. Specific cases involving the lasso and partial smoothness have been analyzed in \cite{bertrand2020implicit} using the weak derivative and in \cite{riis2020geometric, vaiter2017degrees} using the Riemannian gradient. Other specific cases have been worked out for projections onto the cone of positive semidefinite matrices \cite{malick2006clarke}, or solutions of semidefinite programming problems \cite{sun2006strong}, both of which use the Clarke implicit function theorem \cite{clarke1990optimization} to deduce Lipschitz continuity of the solution. From a computational perspective, a software library has been developed in \cite{blondel2021efficient} that can be used for automatic differentiation of implicitly defined functions. As we shall see, our approach is strongly based on the path differentiable implicit function theorem of \cite{bolte2021nonsmooth} which comes with a calculus compatible with the Python library presented in \cite{blondel2021efficient}.%Because the nonsmooth implicit function theorem of \cite[Corollary 1]{bolte2021nonsmooth} shares formally the same formula as the smooth implicit function theorem, the python library presented in \cite{blondel2021efficient} can be used in the nonsmooth case with our framework.

\paragraph{Iterative differentiation/Unrolling} Another growing field in which differentiation of solutions is fundamental is unrolling. In that case one wishes to find a solution of an optimization problem together with its derivative by differentiating  some  optimization algorithm with respect to external parameters. Pioneering works are  \cite{gilbert1992automatic,beck1994automatic} and also \cite{griewank2003piggyback}. In a machine learning context, research has been done for smooth algorithms setting in \cite{pedregosa2016hyperparameter, lorraine2020optimizing, mehmood2020automatic} and in the nonsmooth setting in \cite{bolte2022automatic} for path differentiability, \cite{mehmood2022fixed} for partial smoothness, and \cite{ochs2015bilevel} using a specific Bregman divergence. This is generally treated through ad hoc techniques, using for instance the smooth implicit function theorem approach. In \cite{bolte2022automatic}, the approach is different and closer to ours as it uses the theory of conservative gradients. To understand the deep link this unrolling topic has with our present concerns, one needs to remind that iterates of algorithms are generally defined as solutions to a parametrized optimization problem.  So unrolling offers a wide field of applications for solution's differentiation techniques.  %Convergence of iterative differentiation and implicit differentiation both require the same form of contractivity of the residual equation \jer{[mysterious, why people would know what a residual equation is]}, we study sufficient conditions in a monotone inclusion context. 
Although we illustrate our results on general parametrized problems, let us emphasize that our results could also provide  results for ``unrolling", in the spirit of the iterative differentiation analysis of \cite{bolte2022automatic}.\\% applied to parametric monotone inclusions for Forward-Backward \jer{[why using capitals at Forward-Backward?]} algorithm \cite{lions1979splitting}.\\

Each approach has its benefits, e.g., the variational analytic methods exploiting protodifferentiability are more adapted to giving information about the Lipschitz constant of the solution mapping than what we will propose. The most salient point of our contribution is that we are able to guarantee the path differentiability of the solution, which legitimates, in turn, the use of formal derivatives of the solution. In contrast to nearly all of these works and many others on this subject, this means our approach is compatible with modern differential tools like automatic differentiation and the backpropagation algorithm. This compatibility is achieved by way of a flexible calculus that allows all the usual operations of smooth calculus, in particular, the chain rule for differentiation. In the language of the optimization community, our differential results are qualification-free, and, in terms of differential regularity, everything will boil down to checking that the problem is semialgebraic (or definable).

\paragraph{Our approach and its advantages} The general method we propose is to study the solution mapping $\xstar$ by first rewriting \eqref{eq:introProblem} as a locally Lipschitz fixed-point equation, using ideas from operator splitting methods for nonsmooth optimization \cite{combettes2005signal}. For generalized equations of the form given in \eqref{eq:introProblem}, we can use the resolvent $\mathcal{R}_{\A_\theta}$ to write the \emph{forward-backward} map $\fb$, which a solution must be a fixed point of:
\newq{
\xstar=\fb(\xstar,\theta):=\mathcal{R}_{\A_\theta}(\xstar-\mathcal{B}(\xstar)).
}
With this equation, we can continue by applying the nonsmooth implicit function theorem of \cite{bolte2021nonsmooth} to deduce regularity properties and an expression for the generalized derivative (i.e., the conservative mapping) of the solution mapping $\xstar(\theta)$. 

As has been discussed, the rise in popularity of modern automatic differentiation libraries \cite{abadi2016tensorflow, bradbury2018jax, paszke2017automatic} and their widespread use in machine learning calls for a flexible calculus at the crossroads of mathematics and computer science. For instance, almost the entire field of deep learning crucially relies on using  the renowned backpropagation algorithm to do training. In spite of this, most prior work on this subject has either only considered the smooth case or has ignored non-differentiability issues. Thus a major advantage of the present approach, in contrast to other works, is to provide results that are broadly applicable (e.g., for nonsmooth solutions) but which are also compatible with automatic differentiation.

Besides compatibility with automatic differentiation, another advantage of our work is that the formula given by \cite[Corollary 1]{bolte2021nonsmooth} to compute an implicit conservative gradient is formally the same as the formula used in the smooth case. More specifically, a key feature of \cite[Corollary 1]{bolte2021nonsmooth} that our results will inherit is its coherence with the smooth implicit function theorem - that is to say, elements of the conservative Jacobian associated to the solution mapping $\xstar$ can be rigorously computed by formal differentiation just as in the smooth implicit function theorem.

\paragraph{Main results} Our key technical result is an implicit function theorem for families of contractive Lipschitz continuous equations under path-differentiability assumptions. The simplicity of the contractivity  assumption allows us to derive a wealth of regularity results for parametric strongly monotone problems.
%The implicit function theorem we are going to use requires that the matrices in the conservative Jacobian must be invertible, like in the classical implicit function theorem. While such a condition is often simply assumed to hold true in applications, one important aspect of our work is to provide instead {\em simple} sufficient conditions on the problem data in \eqref{eq:introProblem}. % under which one can guarantee that the invertibility condition is satisfied. 
Our core result (\thmref{thm:stronglyMonotone})  asserts that if $\A_\theta$ or $\B_\theta$ is strongly monotone, then $\xstar$ is path differentiable and there is a formula to compute a conservative Jacobian for it, based on the Clarke Jacobians of $\mathcal{R}_{\gamma\A_\theta}$ and $\B_\theta$. Let us insist on the fact that path differentiability easily follows from semi-algebraic or definable assumptions.

As a consequence, many fundamental parametrized optimization problems can be studied, we provide three important classes of examples.
%. Because this is a general result for maximal monotone operators, it has applications to many problems which can be expressed in the form of \eqref{eq:introProblem}, as demonstrated in the chosen examples: 
\thmref{thm:stronglyConvex} deals with sum composite strongly convex optimization problems, which are ubiquitous in many fields from signal processing \cite{combettes2005signal, day2011inversion, weese1992reliable} to machine learning \cite{bishop1995training, kakade2008complexity}. In the framework of generalized convex duality, in the spirit of the Fenchel-Rockafellar theorem, we provide in \thmref{thm:structuredDuality} a regularity result for dual  solutions to primal together with a calculus. In the min-max setting, under classical assumptions, we study the regularity of parametrize saddle points, this is \thmref{thm:minmax}. To be clear, the reach of our results extends beyond just these selected problems; our results represent a way to analyze solutions to any problem, which can be represented as a parametric monotone inclusion having this additive composite structure, encompassing a significant portion of the convex optimization problems in the literature.

Let us conclude by mentioning the fact the contractivity assumption which is behind our analysis in \lemref{lem:operatorNorm} is sharp in the sense that we are able to provide several counterexamples having apparently similar properties --as semialgebraic problems enjoying quadratic \L ojasiewicz inequalities-- which are not amenable to our optimization framework.

\subsection{Organization of the Paper}
In \secref{sec:background}, we review the necessary background material and notation, mostly regarding convex analysis, conservative calculus, and the nonsmooth implicit function theorem for path differentiable functions.  In \secref{sec:monotone}, we develop results for path differentiability of the solution to parametric monotone inclusion problems, formally stating the monotone inclusion problem we consider and the assumptions needed to ensure path differentiability of the solution mapping. In \secref{sec:functions}, we turn to convex optimization and explore sufficient conditions in terms of properties of the objective function to ensure the solution to a convex optimization problem is path differentiable. In \secref{sec:applications}, we consider some general convex optimization problems and show how to apply the results of \secref{sec:monotone} and \secref{sec:functions} to find expressions for implicit conservative Jacobians associated to the solution mappings. Finally, in \secref{sec:conclusion}, we conclude by noting some alternative formulations of the problem, its fixed point equations, and some other details that could have been chosen differently.

\section{Background}\label{sec:background}
\paragraph{Notation} The set of real numbers will be written as $\R$ and the extended real numbers as $\exR:=\R\cup\{+\infty\}$. We denote the set of $p$-times continuously differentiable functions on a given connected open subset $X \subset \R^n$ by $C^p(X)$ and denote the set of $C^1(X)$ functions whose gradient is Lipschitz continuous by $C^{1,1}(X)$. We will use $\Id_n$ to denote the identity matrix in $\R^{n\times n}$ and $\Id$ to denote the identity mapping. A set-valued map $\A \colon \R^n \rightrightarrows \R^m$, is a function from $\R^n$ to subsets of $\R^m$ (including the empty set). The \emph{graph} of a set-valued mapping $\A \colon \R^n \rightrightarrows \R^m$ will be denoted $\gra{\A} :=\brac{(x,u) \in \R^n \times \R^m \colon u\in \A(x)}$. We denote the operator norm of a linear operator $K:\R^n\to\R^m$ as
\newq{
\norm{K}{\mathrm{op}} := \sup\limits_{v\in\R^n}\frac{\norm{Kv}{}}{\norm{v}{}}
}
and extend this to sets of linear operators $\mathcal{K}=\{K_\omega\}_{\omega\in\Omega}$ in the following way
\newq{
\norm{\mathcal{K}}{\mathrm{op}} := \sup\limits_{K\in\mathcal{K}}\sup\limits_{v\in\R^n}\frac{\norm{Kv}{}}{\norm{v}{}}.
}
\subsection{Convex Analysis}
The following definitions and notations coming from convex analysis are well-known and can be found, for instance, in \cite{bauschke2011convex}.

\begin{definition}[Monotone Operator]
A set-valued mapping $\A\colon \R^n \rightrightarrows \R^n$ is called {\em monotone} if there exists $\alpha\geq 0$ such that for all $(x,u)\in\gra(\A)$ and $(y,v)\in\gra(\A)$,
\newq{
\ip{u-v,x-y}{} \geq \alpha \norm{x-y}{}^2.
}
If $\alpha>0$ then $\A$ is called {\em $\alpha$-strongly monotone}.
\end{definition}
A monotone operator $\A$ is said to be \emph{maximal} if its graph is not contained in any other monotone operator. Recall that the \emph{resolvent} of a maximal monotone operator $\A_\theta\colon\R^n\rightrightarrows\R^n$ is the function $\mathcal{R}_{\A_\theta}\colon\R^n\to\R^n$ defined to be $\mathcal{R}_{\A_\theta}:=(\Id + \A_\theta)^{-1}$. There is a special relationship between closed convex proper functions and maximal monotone operators, which we summarize in the next example.
\begin{example}
Let $f\colon \R^n\to\exR$ be a closed convex proper function, then the \emph{subdifferential} of $f$, $\partial f(x):=\brac{u\colon \forall y\in\R^n, f(y)\geq f(x) + \ip{u, y-x}{}}$, is a maximal monotone operator \cite{rockafellar1970maximal} and the resolvent $\mathcal{R}_{\partial f}$ is the \emph{prox operator}  \cite[Example 23.3]{bauschke2011convex} given by
$\prox_{f}(x):= \argmin\limits_{u\in\R^n} f(u) + \frac{1}{2}\norm{x-u}{}^2$. Additionally, if $f$ is $\alpha$-strongly convex, then $\partial f$ is $\alpha$-strongly monotone \cite[Example 22.3(iv)]{bauschke2011convex}.
\end{example}

\subsection{Conservative calculus and path differentiability}
The following notions generalize the concept of differentiability to locally Lipschitz functions, from Clarke Jacobians to conservative Jacobians. In contrast with  Clarke Jacobians, conservative Jacobians \cite{bolte2020mathematical} offer a way to extend differentiation to locally Lipschitz functions in a way that is compatible with differential calculus.
\begin{definition}[Clarke Jacobian]
The \emph{Clarke Jacobian} of a locally Lipschitz function $f\colon\R^n\to\R^m$ at a point $x$ is defined to be
\begin{equation*}
\cljac_f(x) := \conv\left\{\lim\limits_{k\to+\infty}\jac_f(x_k):x_k\in\diff f, \lim\limits_{k\to+\infty}x_k=x\right\},
\end{equation*}
where $\diff f \subset \R^n$ is the set of full measure where $f$ is differentiable in the classical sense. 
\end{definition}
The following lemma comes from \cite[Proposition 2.1.2(a)]{clarke1990optimization}, it is a nonsmooth generalization of the classical result that a $\beta$-Lipschitz continuous differentiable function has gradient bounded by $\beta$ in norm.
\begin{lemma}[\cite{clarke1990optimization}]\label{lem:clarkeBound}
Let $U\subset \R^n$ be an open set and consider a Lipschitz continuous function $f\colon U\to\R^m$. Then $f$ is Lipschitz continuous with constant $\beta$, if and only if, for all $x\in U$, $\norm{\cljac_f(x)}{\mathrm{op}}\leq \beta$.
\end{lemma}
\begin{definition}[Conservative Jacobian \cite{bolte2020mathematical}]
\label{def:conservativeJacobian}
A {\em conservative Jacobian} for a locally Lipschitz function $f:\R^n\to\R^m$ is a set-valued mapping $\consjac_f\colon \R^n\rightrightarrows\R^{m\times n}$ which is nonempty, locally bounded, graph closed, and satisfies, for any absolutely continuous curve $\gamma\colon[0,1]\to\R^n$,
\begin{equation*}
\forall u\in \consjac_f(\gamma(t)), \quad\frac{d}{dt}f(\gamma(t)) = \ip{u, \dot{\gamma}(t)}{}\ \mbox{for almost all } t \in [0,1]
\end{equation*}
\end{definition}

A function $f:\R^n\to\R^m$ is called \emph{path differentiable} if it is locally Lipschitz and admits a conservative Jacobian $\consjac_f$. This is equivalent to the Clarke Jacobian $\cljac_f$ being conservative for $f$.

\smallskip

Path differentiable functions are ubiquitous among locally Lipschitz functions. The most prominent class of examples is that of semialgebraic functions. For an introduction to the subject of semialgebraic functions, we refer the interested reader to \cite{coste2000introduction}. We simply recall here that a function $f\colon\R^n\to\R^n$ is said to be \emph{semialgebraic} if $\gra{f}$ is a semialgebraic set, i.e., it can be written as the finite union and-or intersection of polynomial equations and inequalities. Let us also mention that all locally Lipschitz definable functions are path differentiable, see \cite{bolte2007clarke} and \cite{bolte2020mathematical}.

\smallskip

Given a path differentiable function $f:\R^{p+n}\to\R^n$, we will write $\cljac_{x,f}(\theta,x):=\{V \colon [U\ V]\in\cljac_{f}(\theta,x)\}\subset \R^{n\times n}$ and refer to this object as the Clarke Jacobian of $f$ with respect to $x$ (and the analog for when a conservative Jacobian has been specified). Similarly, we will write $\cljac_{\theta,f}(\theta,x):=\{U \colon [U\ V]\in\cljac_{f}(\theta,x)\}\subset\R^{n\times p}$ for the Clarke Jacobian of $f$ with respect to $\theta$. A very important but subtle point is that these sets are the projections of the joint Clarke Jacobians, which are possibly distinct from the sets given by fixing $\theta$ and computing the conservative Jacobian or the Clarke Jacobian with respect to $x$ alone.

\medskip

The following nonsmooth implicit differentiation theorem from \cite[Corollary 1]{bolte2021nonsmooth} is the main tool with which we can prove path differentiability and calculate elements of conservative Jacobians of solutions to monotone inclusions. Its main requirements are a path differentiable defining equation $f$ and an invertibility condition on the elements of a conservative Jacobian associated with $f$.

\begin{theorem}[Conservative implicit function theorem \cite{bolte2021nonsmooth}]\label{thm:nonsmoothImplicit}
Let $f:\R^n\times\R^m\to\R^m$ be path differentiable with conservative Jacobian $\consjac_f$. Let $(\hat{x},\hat{y})\in\R^n\times\R^m$ be such that $f(\hat{x},\hat{y})=0$. Assume that $\consjac_f(\hat{x},\hat{y})$ is convex and that, for each $[U\ V ]\in\consjac_f(\hat{x},\hat{y})$, the matrix $V$ is invertible. Then, there exists an open neighborhood $C\times D\subset\R^n\times\R^m$ of $(\hat{x},\hat{y})$ and a path differentiable function $g:C\to D$ such that, for each $x\in C$,
\begin{equation*}
f(x,g(x)) =0
\end{equation*}
and $g$ admits a conservative Jacobian given, for each $x\in C$, by
\begin{equation*}
\consjac_g\colon x\rightrightarrows \{-V^{-1}U:[U\ V]\in\consjac_f(x,g(x))\}.
\end{equation*}
\end{theorem}

In contrast to the accessibility of the formulas involved, the invertibility condition needed to apply \thmref{thm:nonsmoothImplicit} is more difficult to verify than its smooth counterpart due to the fact that the ordinary gradient is a singleton while the conservative gradient is set-valued. Indeed, for smooth implicit differentiation, it suffices to check the invertibility of a single matrix, while in the nonsmooth setting, one is tasked with showing the invertibility of {\em every}  element of a set of matrices. Additionally, the invertibility of the matrix computed in the smooth setting can be checked during runtime while, in the nonsmooth setting, checking the matrix computed at runtime will not be sufficient to ensure that the invertibility condition is holding (see \cite[Section 5]{bolte2021nonsmooth}). For these reasons, it is imperative to have general sufficient conditions outlined which  guarantee that the invertibility condition is satisfied, which is what we develop in the following sections.

\paragraph{Analogy with the Smooth Case}
One way to frame our work is as a study of how the  calculus for solutions to smooth convex parametric optimization persists in nonsmooth convex parametric optimization thanks to conservative Jacobians. To explain this further, let us describe first a typical application of the smooth implicit function theorem to study convex parametric problems. Consider
\newq{
\min\limits_{x\in\R^n}f(\theta,x)
}
where $f\colon\R^p\times\R^n$ is continuously differentiable jointly in $(\theta,x)$, twice-differentiable in $x$, and convex in $x$ for all $\theta\in\R^p$. To examine the existence and regularity of a solution mapping $\xstar$ as a function of $\theta$, one can use the smooth implicit function theorem on the optimality condition $\nabla_x f(\theta,\xstar)=0$ to get $\frac{\partial \xstar}{\partial\theta}(\theta) = -(\nabla_x^2 f(\theta,\xstar(\theta)))^{-1} \frac{\partial}{\partial\theta}\nabla_x f(\theta,\xstar(\theta))$. In addition to the differentiability assumptions we've made, the application of the smooth implicit function here requires the Hessian $\nabla_x^2 f (\theta,\xstar)$ to be invertible.

For a twice-differentiable convex function, the invertibility of the Hessian $\nabla_x^2 f(\theta,x)$ locally around $\xstar$ is equivalent to the strong convexity of the function $f$ locally around $\xstar$, which is itself equivalent to a local quadratic growth condition $f(x)-f(y) - \ip{\nabla f(y),x-y}{} \geq \frac{\rho}{2}\norm{x-y}{}^2$ for some $\rho>0$, for all $x,y$ in some neighborhood of $\xstar$. A byproduct of this paper is to investigate to what extent this trifold equivalency fails to hold when $f$ is no longer assumed to be twice-differentiable. To this end, we are able to show positive and negative results in \secref{sec:functions}: strong convexity is sufficient to ensure the invertibility condition required by the nonsmooth implicit function theorem of \cite[Corollary 1]{bolte2021nonsmooth} holds, meanwhile a local quadratic growth condition is insufficient even when $f$ is a semialgebraic function.

%%%%%%%%%%%%%%%%%%%%%%%%%%%%%%%%%%%%%%%%%%%%%%%%%%%%%%%%%

\section{Solutions to monotone inclusions}\label{sec:monotone}

\subsection{Regularity assumptions and Lipschitz reformulations} 

\paragraph{A Lipschitz reformulation} We begin this section by formally defining the parametric monotone inclusion problem we are considering (whose solution we seek to differentiate) and the assumptions we impose on it. When dealing with parametrized mappings like $\A\colon\R^m\times\R^n\to\R^n$, it will be convenient to use subscript notation $\A_\theta$ to denote the mapping corresponding to $\A(\theta,\cdot)\colon\R^n\to\R^n$ --~this notation will be used frequently throughout the rest of the paper.

\begin{assumption}[Path differentiability]\label{ass:pathDifferentiable}{\rm Let $\Theta \subset \R^p$ be a nonempty connected open set, $\A\colon \Theta\times\R^n\to\R^n$, $\B\colon\Theta\times\R^n\to\R^n$, and $\gamma>0$ be a stepsize. For all $\theta\in\Theta$, assume the following two conditions hold,
    \begin{enumerate}
        \item $\A(\theta,\cdot):\R^n\rightrightarrows\R^n$ is maximal monotone and $\B(\theta,\cdot):\R^n\to\R^n$ is maximal monotone and locally Lipschitz. Furthermore the solution set to the following inclusion is nonempty
        \begin{equation}\label{eq:monotoneInclusion}\tag{$\mathscr{P}_{\mathrm{mono}}$}
            0 \in \A(\theta, \cdot) + \B(\theta, \cdot).
        \end{equation}
        \item The resolvent $\mathcal{R}_{\gamma \A_\theta}$ and the map $\B$ are both locally Lipschitz and path differentiable, jointly in $(\theta,x)$, so that the function 
        \begin{equation}
        \fb(\theta,x) := \mathcal{R}_{\gamma\A_\theta}(x-\gamma \B_\theta(x))
        \end{equation}
        is path differentiable jointly in $(\theta,x)$.    	    
    \end{enumerate}
}\end{assumption}

There are many different fixed-point reformulations 
 of \eqref{eq:monotoneInclusion} one can choose from, each one inducing a function $\fb_\theta$, which we discuss in \secref{sec:conclusion}. The Lipschitz mapping $\fb_\theta$  we have opted for is reminiscent of the forward-backward splitting algorithm \cite{lions1979splitting}. It is general enough to cover a variety of monotone inclusions coming from both smooth and nonsmooth convex optimization problems. Similar to the choice of $\fb_\theta$, the choice of the constant $\gamma$ in \assref{ass:pathDifferentiable} is  also arbitrary provided $\gamma$ is positive. Note also that the solution $\xstar$ does not depend on $\gamma$, although it can be defined as a fixed point of $\fb_\theta$, which depends on $\gamma$. In each theorem, this constant will be finely tuned so as to obtain the most general regularity results possible. It is important to understand here that the regularity properties we will obtain in the following sections may depend on the reformulation in \assref{ass:pathDifferentiable} that we have chosen. Whether it is the case or not is a matter for future research.

%We will be interested in differential properties of $\xstar$ and intuition from implicitly differentiating solutions to smooth unconstrained optimization problems leads one to believe that the derivative of $\xstar$ should not have too much dependence on $\gamma$. In a computational context, one can always choose $\gamma > 0$, we leave a precise study of the dependency on $\gamma$ of the theory developed further in this work for future research.

\medskip

%Besides formally defining \eqref{eq:monotoneInclusion}, \assref{ass:pathDifferentiable} serves as a foundation for our later arguments because of the path differentiability required by \thmref{thm:nonsmoothImplicit}. 
%Although we do not specify here that the solution to \eqref{eq:monotoneInclusion} is unique, the conditions we impose later on will be sufficient to guarantee this, so that we will be effectively assuming the solution is unique.  In addition to path differentiability, the conservative implicit function theorem \thmref{thm:nonsmoothImplicit} also requires an invertibility condition for every element of the conservative Jacobian of $F$. Since the conservative Jacobian associated to a path differentiable function is not unique, we must fix a particular choice of conservative Jacobian, which we will make specific in \defref{def:jacobianChoice}. 

%To be clear about the function that we are applying the implicit function theorem to in what follows, 
We define the \emph{residual function} $\res:\R^{p+n}\to\R^n$ to be 
\begin{equation}\label{eq:residualFunction}
\res(\theta,x):= x - \fb_\theta(x) \tag{$\mathscr{P}_{\mathrm{mono-bis}}$}
\end{equation}
%We have therefore formulated the inclusion problem $\mathscr{P}_{\mathrm{mono}}$ as a classical 
%Lipchitz equation $\mathscr{P}_{\mathrm{mono-bis}}$. 
so that
\begin{equation}\res(\theta,x^*(\theta))=0,
\end{equation}
whenever the expression is well-defined. 

An essential fact is that this equation will be automatically path differentiable if its constituents are semialgebraic, or more generally definable \cite{coste2000introduction, bolte2007clarke, bolte2020mathematical}. Classically, this will also imply that the solution mapping is semialgebraic (as a set-valued mapping).

\paragraph{First-order properties of the residual equation}
Let us describe the  key first-order objects that will ensure the existence of solution maps and allow us to establish their path-differentiability. We work under \assref{ass:pathDifferentiable}.

Set $T\colon\Theta\times\R^n\to\R^{n}$ to be $T(\theta,x):= \mathcal{R}_{\gamma\A_\theta}(x)$ and consider the two following conservative Jacobians relative to the path differentiable mapping $H$:
\begin{align}
&\consjac_{\fb}(\theta,x)
&&=\{\begin{bmatrix}U-\gamma VW & V(\Id_n-\gamma Z)\end{bmatrix} \colon [U\ V]\in\cljac_T(\theta,x-\gamma\B_\theta(x)), [W\ Z]\in\cljac_{\B}(\theta,x)\}.\notag
 \\
&\consjac_{x,\fb}(\theta,x)
&&= \cljac_{x,T}(\theta,x-\gamma\B_\theta(x))\times (\Id_n-\gamma\cljac_{x,\B}(\theta,x)). \label{def:jacobianChoice}\end{align}

The first set valued map $\consjac_{\fb}$ given in  \eqref{def:jacobianChoice} is simply obtained by a formal application of the rules of differential calculus to the composite structure of $H$ in \assref{ass:pathDifferentiable} using Clarke Jacobians instead of classical Jacobians, while the second is a partial derivative version obtained by mere projection.\\
 
 \noindent
Indeed, define $S\colon \Theta\times\R^n\to\Theta\times\R^n$ to be $S(\theta,x):= (\theta,x-\gamma\B(\theta,x))$  so that $T(S(\theta,x)) = \fb(\theta,x)$. One can check that $\consjac_{\fb}$ is the following product of Clarke Jacobians for $(\theta,x)\in\Theta\times\R^n$
\newq{
\consjac_{\fb}(\theta,x) &= \cljac_{T}(S(\theta,x))\times \cljac_S(\theta,x)\\
&=\{\begin{bmatrix}U & V\end{bmatrix}\times\begin{bmatrix}\Id_p & 0\\ -\gamma W & \Id_n-\gamma Z\end{bmatrix} \colon [U\ V]\in\cljac_{T}(\theta,x-\gamma\B_\theta(x)), [W\ Z]\in\cljac_{\B}(\theta,x)\}\\
&=\{\begin{bmatrix}U-\gamma VW & V(\Id_n-\gamma Z)\end{bmatrix} \colon [U\ V]\in\cljac_T(\theta,x-\gamma\B_\theta(x)), [W\ Z]\in\cljac_{\B}(\theta,x)\}
}
which is a conservative Jacobian for $\fb$. Consequently we have the following conservative Jacobian with respect to $x$ for $\fb$, for each fixed $(\theta,x)\in\Theta\times\R^n$,
\newq{
\consjac_{x,\fb}(\theta,x) &= \{V(\Id_n-\gamma Z) \colon [U\ V]\in\cljac_{T}(\theta,x-\gamma\B_\theta(x)), [W\ Z]\in\cljac_{\B}(\theta,x)\}\\
&= \cljac_{x,T}(\theta,x-\gamma\B_\theta(x))\times (\Id_n-\gamma\cljac_{x,\B}(\theta,x)),
}
which concludes our explanation.\\

\smallskip

Let us now define a conservative Jacobian for the residual function $\res$ from \eqref{eq:residualFunction} through:
\newq{
\consjac_{\res}(\theta,x) = \{\begin{bmatrix}\gamma VW-U & \Id_n-V(\Id_n-\gamma Z)\end{bmatrix} \colon [U\ V]\in\cljac_T(\theta, x-\gamma\B_\theta(x)), [W\ Z]\in \cljac_{\B}(\theta,x)\},
}
for each $(\theta,x)\in\Theta\times\R^n$.
\begin{remark}\label{rem:choiceOfJacobian}
The choice of $\consjac_{x,\fb}(\theta,x)$ in \eqref{def:jacobianChoice} corresponds to the set-valued object computed by applying formal differentiation, i.e., the chain rule,  to $\fb$ as a composition of functions. Additionally, the particular form of $\consjac_{x,\fb}(\theta,x)$ we consider above allows one to control $\norm{\consjac_{x,\fb}(\theta,x)}{\mathrm{op}}$ when one of the monotone operators $\A_\theta$ or $\B_\theta$ is strongly monotone, as we will expose later. In terms of the residual function $\res$ from \eqref{eq:residualFunction}, the choice made in \eqref{def:jacobianChoice} induces a particular conservative Jacobian $\consjac_{\res}$ and thus also a particular conservative Jacobian with respect to $x$,
\newq{
\consjac_{x,\res}(\theta,x) &= \{\Id_n-V(\Id_n-\gamma Z) \colon [U\ V]\in\cljac_T(\theta, x-\gamma\B_\theta(x)), [W\ Z]\in \cljac_{\B}(\theta,x)\}\\
&= \Id_n - \cljac_{x,T}(\theta,x-\gamma\B_\theta(x))\times (\Id_n-\gamma\cljac_{x,\B}(\theta,x)).
}
\end{remark}

\subsection{Contractivity implies the path differentiatiability of solutions}

With this particular choice of conservative Jacobians for $\fb$ and $\res$ given by \eqref{def:jacobianChoice}, and under the additional assumption of strong monotonicity, we  show below that a simple contractivity condition holds.

\begin{definition}[Contractivity of residual equations]\label{cond:contractivity} Under \assref{ass:pathDifferentiable}, we shall say that the residual equation $\res=0$ is {\em contractive}, or has the {\em contractivity property}, if for each $\theta\in\Theta$, there is a solution $\xstar(\theta)$ to \eqref{eq:monotoneInclusion} with $$\norm{\consjac_{x,\fb}(\theta,\xstar(\theta))}{\mathrm{op}} < 1$$
where we recall that $H=\Id-\res$.\\
\end{definition}

Contractivity in \defref{cond:contractivity} is key to applying nonsmooth implicit path differentiation, first because it actually warrants the well-posedness and the single-valuedness of the solution mapping (\cite[Proposition 26.1]{bauschke2011convex}, see as well \cite[Theorem 11]{hiriart-urruty1979tangent}), 
and also because it is closely related to the invertiblity condition required in the conservative implicit function theorem  (\thmref{thm:nonsmoothImplicit}).

The next lemma shows that contractivity is key to our approach to solutions of monotone inclusions.

\begin{lemma}[Path differentiability of the solution map]\label{lem:operatorNorm}
Under \assref{ass:pathDifferentiable} and assuming $\,\res$ has the contractivity property, then $\xstar$ is unique, path differentiable on $\Theta$, and admits a conservative Jacobian of the form 
\newq{
&\consjac_{\xstar} \colon \theta\rightrightarrows\\
&\{\para{\Id_n-V(\Id_n-\gamma Z)}^{-1}\para{U-\gamma VW} \colon [U\ V]\in\cljac_T(\theta, \xstar(\theta)-\gamma\B_\theta(\xstar(\theta))), [W\ Z]\in \cljac_{\B}(\theta,\xstar(\theta))\}.
}
\end{lemma}
\begin{proof}
We begin with the uniqueness of $\xstar$ for each $\theta\in\Theta$. By convexity of the operator norm, it holds
\newq{
\norm{\conv\para{\consjac_{x,\fb}(\theta,\xstar(\theta))}}{\mathrm{op}}\leq\norm{\consjac_{x,\fb}(\theta,\xstar(\theta))}{\mathrm{op}} < 1.
}
However, $\cljac_{x,\fb}(\theta,\xstar(\theta)) \subset \conv\para{\consjac_{x,\fb}(\theta,\xstar(\theta))}$ and so $\norm{\cljac_{x,\fb}(\theta,\xstar(\theta))}{\mathrm{op}} < 1$. From this we conclude that $\fb$ is locally a strict contraction around $(\theta,\xstar(\theta))$ and thus the solution $\xstar(\theta)$ is unique.

We will apply  \thmref{thm:nonsmoothImplicit} to $\res$ using the pointwise convex hull $(\theta,x) \rightrightarrows \conv(\consjac_{\res}(\theta,x))$ as a conservative Jacobian for $\res$. Note that it is indeed conservative since $\consjac_{\res}$ is and pointwise convex hulls preserves  \defref{def:conservativeJacobian}. We will use the shorthand notation $\conv(\consjac_{x,\res})$ to denote for each fixed $\theta \in \Theta$ the set valued map  $(\theta,x) \rightrightarrows \conv(\consjac_{\res}(\theta,x))$, and similarly for $\fb$. Note that $\conv(\consjac_{x,\res}) = \Id_n - \conv(\consjac_{x,\fb})$.

%We aim to use \thmref{thm:nonsmoothImplicit} applied to the equation $F(\theta,\xstar(\theta))=0$ for each $\theta\in\Theta$. 
%The convexity required by \thmref{thm:nonsmoothImplicit} is negligible since one can take the convex hull of the conservative Jacobian defined in \defref{def:jacobianChoice} without loss of generality. 
%The invertibility condition of \thmref{thm:nonsmoothImplicit} applied to this problem requires every element of $\consjac_{x,F}(\theta,\xstar(\theta))$ to be invertible. 
Fix $\theta \in \Theta$. Let us show that the contractivity condition entails that every element of $\conv(\consjac_{x,\res}(\theta,\xstar(\theta)))$ is invertible. Indeed, 
set $\rho = \norm{\conv(\consjac_{x,\fb}(\theta,\xstar(\theta)))}{\mathrm{op}} = \norm{\consjac_{x,\fb}(\theta,\xstar(\theta))}{\mathrm{op}}<1$ (use triangle inequality), so that for any $M_{\res} \in \conv(\consjac_{x,\res}(\theta,\xstar(\theta)))$, there is $M_\fb \in \conv(\consjac_{x,H}(\theta,\xstar(\theta)))$ such that $M_{\res} = \Id_n - M_\fb$, and for all $v \in \R^n$, we have
\newq{
\norm{M_{\res} v}{}&= \norm{(\Id_n - M_\fb) v}{} \\
&\geq \norm{v}{} - \norm{M_\fb v}{} \\
& \geq \norm{v}{} (1 - \rho)
}
% \newq{
% \norm{\consjac_{x,F}(\theta,\xstar(\theta))}{\mathrm{op}} \geq \absv{\norm{\Id_n}{\mathrm{op}} - \norm{\consjac_{x, \fb_\theta}}{\mathrm{op}}} = \absv{1- \norm{\consjac_{x, \fb_\theta}(\xstar(\theta))}{\mathrm{op}}}>0.
% }
which shows that $M_{\res}$ is invertible because $\rho < 1$.
Since $\res$ is path differentiable and all of the elements of $\consjac_{x,\res}(\theta,\xstar(\theta))$ are invertible for each $\theta\in\Theta$, the conditions to apply \thmref{thm:nonsmoothImplicit} to the equation $\res(\theta,\xstar(\theta))=0$ hold, and thus $\xstar$ is path differentiable on $\Theta$. The formula for the conservative Jacobian follows from \thmref{thm:nonsmoothImplicit} because it defines a graph closed and locally bounded set valued map which is a subset of the set valued map obtained by applying \thmref{thm:nonsmoothImplicit} to $\conv(\consjac_{\res})$ which satisfies the chain rule of \defref{def:conservativeJacobian}. 
% By \assref{ass:pathDifferentiable} with \defref{def:jacobianChoice}, the residual function $F(\theta,x)$ defined in \eqref{eq:residualFunction} is path differentiable and has a conservative Jacobian with respect to $x$ given for each $(\theta,x)\in\Theta\times\R^n$ by
% \newq{
% \consjac_{x,F}(\theta,x) &= \Id_n-\consjac_{x, \fb_\theta}(\theta,x) \\
% &= \{\Id_n-B(\Id_n-\gamma D) \colon [A\ B]\in\cljac_T(\theta, \xstar(\theta)-\gamma\B_\theta(\xstar(\theta))), [C\ D]\in \cljac_{\B}(\theta,\xstar(\theta))\}.
% }
\end{proof}

The expression for the conservative Jacobian of $\xstar$ given in \lemref{lem:operatorNorm} can be more compactly expressed in terms of the conservative Jacobian of $\fb$ defined in \eqref{def:jacobianChoice}, for each $\theta\in\Theta$,
\newq{
\consjac_{\xstar}\colon \theta\rightrightarrows \{\para{\Id_n - V}^{-1}U\colon[U\ V]\in\consjac_{\fb}(\theta,\xstar(\theta))\}.
}

\subsection{Strongly monotone inclusions have path differentiable solutions}
The following theorem is related to \cite[Prop. 26.16]{bauschke2011convex}, which provides sufficient conditions for linear convergence of the forward-backward algorithm applied to finding a zero of the sum of two maximally monotone operators $\A$ and $\B$. It is however important to observe that linear convergence is not enough to reach the same conclusions (see Section~\ref{s:beyond}).

% In the process of proving linear convergence, they also show that the forward-backward operator is a contraction; this is the key argument we will adapt to our needs.

\begin{theorem}[Path differentiability: strongly monotone case]\label{thm:stronglyMonotone}
Under \assref{ass:pathDifferentiable}, consider \eqref{eq:monotoneInclusion} and, for each $\theta\in\Theta$, assume that $\B_\theta$ is $\beta$-Lipschitz continuous and that either $\A_\theta$ or $\B_\theta$ is $\alpha$-strongly monotone, for some $\alpha,\beta >0$, uniformly in $\theta$. 

Then, for $\gamma\in\left(0,\frac{2\alpha}{(\alpha + \beta)^2}\right)$, the residual map $\res$ of  \eqref{eq:monotoneInclusion} is contractive, i.e., the inequality in \defref{cond:contractivity} holds. Furthermore, $\xstar$ is unique and path differentiable on $\Theta$ with a conservative Jacobian given for each $\theta\in\Theta$ by
\newq{
&\consjac_{\xstar} \colon \theta\rightrightarrows\\
&\{\para{\Id_n-V(\Id_n-\gamma Z)}^{-1}\para{U-\gamma VW} \colon [U\ V]\in\cljac_T(\theta, \xstar(\theta)-\gamma\B_\theta(\xstar(\theta))), [W\ Z]\in \cljac_{\B}(\theta,\xstar(\theta))\}.
}
\end{theorem}
\begin{proof}
Under \assref{ass:pathDifferentiable} with the conservative Jacobians given in \eqref{def:jacobianChoice}, the forward-backward mapping $\fb$ is path differentiable on $\Theta\times\R^n$ with a conservative Jacobian with respect to $x$ given by
\begin{equation*}
\consjac_{x,\fb}(\theta,x)=\cljac_{x,T}(\theta,x-\gamma \B_\theta (x))\times (\Id_n-\gamma\cljac_{x,\B}(\theta,x)).
\end{equation*}
We take an arbitrary $\theta\in\Theta$ and divide the proof into cases depending on whether $\A_\theta$ or $\B_\theta$ is $\alpha$-strongly monotone.

Assume that $\B_\theta$ is $\alpha$-strongly monotone with $\alpha \leq \beta$. Since $\alpha>0$ and $0<\gamma < \frac{2\alpha}{(\alpha + \beta)^2} < \frac{2\alpha}{\beta^2}$, it holds that $\gamma(2\alpha - \gamma\beta^2)>0$, and furthermore $\gamma(2\alpha - \gamma \beta)\leq \gamma 2 \alpha < \frac{4 \alpha^2}{(\alpha + \beta)^2} \leq 1$. Setting $\tau = \sqrt{1-\gamma(2\alpha-\gamma\beta^2)}$, %which is a positive real number due to the fact that $\alpha\leq\beta$ and $\gamma\in(0,\frac{2\alpha}{(\alpha+\beta)^2})$, 
we have that $\mathcal{R}_{\gamma \A_\theta}$ is nonexpansive \cite[Proposition 23.8]{bauschke2011convex} and $\Id-\gamma \B_\theta$ is $\tau$-Lipschitz continuous with $0 \leq \tau < 1$. Thus, by applying \lemref{lem:clarkeBound} for each $(\theta,x)\in\Theta\times\R^n$,
\newq{
\norm{\consjac_{x,\fb}(x)}{\mathrm{op}} \leq \norm{\cljac_{x,T}(\theta,x-\gamma \B_\theta (x))}{\mathrm{op}} \norm{(\Id_n-\gamma\cljac_{x,\B}(\theta,x))}{\mathrm{op}}\leq\tau<1.
}

Now we consider case where $\A_\theta$ is $\alpha$-strongly monotone. We have by \cite[Proposition 23.13]{bauschke2011convex} that the resolvent $\mathcal{R}_{\gamma\A_\theta}$ is Lipschitz continuous with constant $\frac{1}{1+\gamma \alpha}$ and the mapping $\Id-\gamma \B_\theta$ is $\sqrt{1+\gamma^2\beta^2}$-Lipschitz continuous by \lemref{lem:lipschitzConstant}, giving for each $(\theta,x)\in\Theta\times\R^n$,
\newq{
\norm{\consjac_{x,\fb}(x)}{\mathrm{op}}=\norm{\cljac_{x,T}(\theta,x-\gamma \B_\theta (x))}{\mathrm{op}} \norm{(\Id_n-\gamma\cljac_{x,\B}(\theta,x))}{\mathrm{op}}\leq\frac{\sqrt{1+\gamma^2\beta^2}}{1+\gamma\alpha}.
} 
% We can further divide this case into two subcases based on $\alpha$ and $\beta$; we consider first the subcase where $\alpha >\beta$, so that
% \newq{
% \norm{\consjac_{x,\fb_{\theta}}(x)}{\mathrm{op}}\leq\frac{\sqrt{1+\gamma^2\beta^2}}{1+\gamma\alpha}<\frac{1+\gamma\beta}{1+\gamma\alpha}<1.
% }

% On the other hand, consider the subcase where $\alpha \leq \beta$, 
Since $\gamma\in\left(0,\frac{2\alpha}{(\alpha + \beta)^2}\right)$, it holds
\newq{
%\gamma < \frac{2\alpha}{(\alpha + \beta)^2} &<\frac{2\alpha}{\beta^2-\alpha^2}\\
%\gamma(\beta^2-\alpha^2)&<2\alpha\\
1+\gamma^2\beta^2 &< 1 + \gamma (\gamma (\alpha+\beta)^2)<1 + 2 \alpha \gamma < 1 + \gamma2\alpha + \gamma^2\alpha^2 =  (1+\gamma\alpha)^2,\\
%\gamma^2(\beta^2-\alpha^2)&< \gamma^2(\beta + \alpha)^2 <2\alpha \gamma\\
%\gamma^2\beta^2&<2\alpha\gamma + \gamma^2\alpha^2\\
%1+\gamma^2\beta^2 &< 1 + \gamma2\alpha + \gamma^2\alpha^2 \\
%\sqrt{1+\gamma^2\beta^2}&<1+\gamma\alpha
}
so that $\frac{\sqrt{1+\gamma^2\beta^2}}{1+\gamma\alpha}<1$. Putting everything together we find, for each $(\theta,x)\in\Theta\times\R^n$,
\newq{
\norm{\consjac_{x,\fb_{\theta}}(x)}{\mathrm{op}}=\norm{\cljac_{x,T}(\theta,x-\gamma \B_\theta (x))}{\mathrm{op}} \norm{(\Id_n-\gamma\cljac_{\B}(\theta,x))}{\mathrm{op}}\leq\frac{\sqrt{1+\gamma^2\beta^2}}{1+\gamma\alpha}<1.
}

We have established that $\norm{\consjac_{\fb_{\theta}}(x)}{\mathrm{op}}< 1$ for all $(\theta,x)\in\Theta\times\R^n$ in both cases of the theorem, and we have contractivity. By \lemref{lem:operatorNorm}, $\xstar$ is, therefore, path differentiable on $\Theta$ and the desired formula for the conservative Jacobian follows.
\end{proof}

\begin{remark}[On the constant $\gamma$]
The restriction on the values that $\gamma$ can take in the different cases of \thmref{thm:stronglyMonotone} can be relaxed if more information about the operators $\A_\theta$ and $\B_\theta$ is specified. For instance, if $\B_\theta$ is $\beta$-cocoercive rather than $\beta$-Lipschitz and $\A_\theta$ is $\alpha$-strongly monotone, then $\fb_\theta$ is a contraction for any $\gamma\in(0,2\beta)$. It is important to notice that the choice of $\gamma$ for implicit differentiation need not match the $\gamma$ chosen for solving the problem (indeed, the algorithm to solve the problem and the fixed point equation for optimality need not match to begin with).
\end{remark}

%%%%%%%%%%%%%%%%%%%%%%%%%%%%%%%%%%%%%%%%%%%%%%%%%

\section{Path differentiation of solutions to convex optimization problems}\label{sec:functions}

Let $\Theta\subset\R^p$ be a connected open set and consider, for each $\theta\in\Theta$, the parametric optimization problem of finding a minimizer
\begin{equation}\label{eq:optimizationProblem}\tag{$\mathscr{P}_{\mathrm{opt}}$}
\xstar:=\argmin\limits_{x\in\R^n}f(\theta,x) + g(\theta,x)
\end{equation}
where $f_\theta := f(\theta,\cdot)\in C^{1,1}(\R^n)$ is convex and $g_\theta:=g(\theta,\cdot)$ is a closed convex proper function from $\R^n$ to $\exR$. It is well known \cite[Theorem 26.2]{bauschke2011convex} that this problem is equivalent to finding a zero $\xstar$ of the sum of two monotone operators given by the subdifferentials of the functions,
\newq{
0\in \nabla f_\theta (\xstar) + \partial g_\theta (\xstar).
}
In this way, the problem of differentiating a solution of \eqref{eq:optimizationProblem} is equivalent to the problem of the previous section - differentiating a solution to a monotone inclusion \eqref{eq:monotoneInclusion}. This equivalence motivates the following assumptions on \eqref{eq:optimizationProblem}, which are analogous to \assref{ass:pathDifferentiable} with the conservative Jacobians defined in \eqref{def:jacobianChoice} for the case where the monotone operators $\A_\theta$ and $\B_\theta$ are subdifferentials of closed convex proper functions.
\begin{assumption}\label{ass:functionPathDifferentiable}
{\rm Let $\Theta$ be a connected open set and let $\gamma >0$. For all $\theta\in\Theta$, let $f_\theta:=f(\theta,\cdot)\in C^{1,1}(\R^n)$ and $g_\theta:=g(\theta,\cdot)$ be closed convex proper functions from $\R^n$ to $\exR$ and assume that the prox operator $\prox_{\gamma g_\theta}\colon\Theta\times\R^n\to\R^n$ and the gradient $\nabla f_\theta\colon\Theta\times\R^n\to\R^n$ are both path differentiable, jointly in $(\theta,x)$.}
\end{assumption}

A sufficient condition guaranteeing the path differentiability in \assref{ass:functionPathDifferentiable} holds is to assume that $f$ and $g$ are semialgebraic functions. Under \assref{ass:functionPathDifferentiable}, one can consider $\A_\theta = \partial g_\theta$ and $\B_\theta = \nabla f_\theta$ so that \assref{ass:pathDifferentiable} is met and $\fb$ is the forward-backward algorithm applied to \eqref{eq:optimizationProblem}. Using the conservative Jacobians defined in \eqref{def:jacobianChoice}, we have in this case for all $(\theta,x)\in\Theta\times\R^n$
\begin{align}
    \consjac_{\fb}(\theta,x) &=\{\begin{bmatrix}U-\gamma VW & V(\Id_n-\gamma Z)\end{bmatrix} \colon [U\ V]\in\cljac_{\prox_{\gamma g_\theta}}(\theta,x-\gamma\nabla f_\theta(x)), [W\ Z]\in\cljac_{\nabla f_\theta}(\theta,x)\}\nonumber\\
    \consjac_{x,\fb}(\theta,x) &= \cljac_{x,\prox_{\gamma g_\theta}}(\theta,x-\gamma\nabla f_\theta(x))\times (\Id_n-\gamma\cljac_{x,\nabla f_\theta}(\theta,x)).
    \label{eq:functionJacobianChoice}
\end{align}

For the moment, we do not explicitly assume that the solution $\xstar$ to \eqref{eq:optimizationProblem} is unique for each $\theta\in\Theta$; the results in later sections will make stronger assumptions that imply the uniqueness of $\xstar$ as a byproduct. We shall also provide assumptions on $g_\theta$ and $f_\theta$ that will ensure the invertibility condition of \thmref{thm:nonsmoothImplicit} holds at the solution $\xstar(\theta)$.

\subsection{Solutions of strongly convex problems}
Recall that the subdifferential of a strongly convex function is strongly monotone \cite[Example 22.4]{bauschke2011convex}. As a consequence of \thmref{thm:stronglyMonotone} for strong monotonicity, we can then formulate the following analgous result for \eqref{eq:optimizationProblem} with strong convexity of $f_\theta$ or $g_\theta$.
\begin{theorem}[Path differentiability: strongly convex case]\label{thm:stronglyConvex}
Let \assref{ass:functionPathDifferentiable} hold with the conservative Jacobians given in \eqref{eq:functionJacobianChoice} and consider \eqref{eq:optimizationProblem} for $\theta\in\Theta$. Denote $\beta>0$ a Lipschitz constant of $\nabla f_\theta$ which is assumed to be uniform in $\theta$. Assume that either $f_\theta$ or $g_\theta$ is $\alpha$-strongly convex with $\alpha>0$ which is also assumed to be uniform in $\theta$.
  Then, for $\gamma = \frac{\alpha}{(\alpha+\beta)^2}$, the solution $\xstar(\theta)$ is unique for each $\theta\in\Theta$ and path differentiable on $\Theta$ with a conservative Jacobian given by
\newq{
&\consjac_{\xstar}\colon\theta\rightrightarrows \{\para{\Id_n - V(\Id_n - \gamma Z)}^{-1}\para{U-\gamma VW}\}
}
where $\sbrac{U\ V}$ range in $\cljac_{T}(\theta,\xstar(\theta)-\gamma\nabla f_\theta(\xstar(\theta)))$ and $ \sbrac{W\ Z}$ in $\cljac_{\nabla f_\theta}(\theta,\xstar(\theta))$.
\end{theorem}
\begin{proof}
Due to \assref{ass:functionPathDifferentiable}, the function $\res(\theta,x)$ is path differentiable on $\Theta\times\R^n$; fix an arbitrary $\theta\in\Theta$. Since $f_\theta\colon\R^n\to\R$ and $g_\theta\colon\R^n\to\exR$ are closed convex proper functions, $\nabla f_\theta$ and $\partial g_\theta$ are maximal monotone operators \cite[Example 22.4]{bauschke2011convex}. Furthermore, since one of $f_\theta$ or $g_\theta$ is $\alpha$-strongly convex, one of the operators $\nabla f_\theta$ or $\partial g_\theta$ is $\alpha$-strongly monotone and \thmref{thm:stronglyMonotone} can be applied with $\A_\theta = \partial g_\theta$ and $\B_\theta = \nabla f_\theta$, yielding the path differentiability of $\xstar$ and the formula for its conservative Jacobian on $\Theta$.
\end{proof}

The choice of $\gamma$ can be relaxed to be any value in $\left(0,\frac{2\alpha}{(\alpha+\beta)^2}\right)$ without issue, we have simply taken $\frac{\alpha}{(\alpha+\beta)^2}$ for convenience. In contrast to some prior work, e.g., \cite{fadili2018sensitivity}, we allow for both $f_\theta$ and $g_\theta$ to be parametrized functions, rather than just one or the other.

\subsection{Beyond strong convexity?}\label{s:beyond}

Strong convexity may be generalized by means of quadratic growth conditions, such as global error bounds \cite{pang1997error} or equivalently KL inequality \cite{bolte2017error}. On the other hand, quadratic \L ojasiewicz inequality, quadratic error bound, turn out to be equivalent to linear convergence of forward-backward iterations under mild conditions; one may consult, for instance, \cite{garrigos2017convergence}. Since contractivity obviously implies linear convergence, it is tempting to think that it could be somehow inserted into the equivalence chain. This would be a natural path beyond strong convexity assumptions.%, for instance, by using quadratic \L ojasiewicz inequalities.}

%Given that the contractivity condition is connected to strong monotonicity and convexity, it is natural to think that it could be replaced by a quadratic \L ojasiewicz inequality, by a quadratic error bound or by a linear convergence assumption on the residual function. \jer{One may consult for instance \cite{} for these types of equivalences.}

%The results so far have used \condref{cond:contractivity}, which implies that the forward backward map, $\fb$, is a strict contraction and in turns implies linear convergence of the algorithm to its unique fixed point by Banach-Picard fixed point theorem. 
This calls, for instance, for the following question:   Does the contractivity of $\res$ always hold if $f$ and $g$ are such that $H$ is linearly convergent to a unique fixed point? An element of motivation is that in the smooth case, when $f$ is $C^2$ and $g = 0$, contractivity is indeed equivalent to \eqref{eq:contractionFB} as discussed in \secref{sec:background}, so that the questions relates to extension of this equivalence to the nonsmooth setting. %\jer{(this paragraph is still obscure by a lack of ref and explanations)}

\smallskip

Linear convergence of the forward-backward mapping can simply be formulated as follows: for a fixed $\theta \in \Theta$ there exists $\rho \in (0,1)$, such that, for all $x$,
\begin{align}
    \norm{\fb(\theta,x) - \xstar(\theta)}{} \leq \rho \norm{x - \xstar(\theta)}{}.
    \label{eq:contractionFB}
\end{align}
 We provide below  two examples having this property while being non strongly convex  contradicting contractivity of $\res$.  The first one has $C^{1,1}$ objective ($g = 0$), while the second  is nonsmooth \propref{prop:semialgebraicProx} ($f = 0$). This answers negatively to the above question.
 
 Let us start with the differentiable case for which $\fb$ reduces to a gradient step. 
\begin{proposition}[Linear convergence does not imply contractivity I]\label{prop:semialgebraic}
There exists a convex semialgebraic function $h\in C^{1,1}(\R^2)$ with $\nabla h$ $1$-Lipschitz and $\|x - \nabla h(x)\| \leq \rho \|x\|$ for all $x$ for some $\rho \in (0,1)$, such that $h$ is not strongly convex locally around $0$ and $\cljac_{\nabla h}(0)$ contains singular matrices.
\end{proposition}
\begin{proof}
Let $Q\subset\RR^2$ be a closed convex set with $0\in\mathrm{int}(Q)$ and a smooth boundary so that there is a differentiable outer pointing unit normal vector $\hat{n}\in C^1(\bd (Q))$. We consider the gauge function $\Psi_Q$ associated to $Q$
\newq{
\Psi_Q(x) = \inf\{\lambda > 0 : x\in\lambda Q\}.
}
The gauge function $\Psi_Q$ \cite[Example 8.36]{bauschke2011convex} is the unique positively homogeneous function such that the sublevel set of level $1$ is $Q$ \cite[Corollary 14.13]{bauschke2011convex}. Furthermore, the sublevel sets of $\Psi_Q$ of level $\lambda \geq 0$ are equal to $\lambda Q$. 

For $x \in \RR^2$, we extend $\hat{n}(x)$ to be the outer pointing normal vector to the set $\Psi_Q(x) Q$ at $x$, it defines a $C^1$ function on $\R^2$. The gradient of $\Psi_Q$ for $x \neq 0$ has to be of the form $\alpha(x) \hat{n}(x)$ for a positive function $\alpha$. By homogeneity, we have for small $t$
\begin{align*}
				\frac{\Psi_Q\left( x \left( 1 + \frac{t}{\Psi_Q(x)} \right)\right) - \Psi_Q(x)}{t} = 1 = \left\langle \frac{x}{\Psi_Q(x)}, \nabla \Psi_Q(x)\right\rangle = \alpha(x) \left\langle\frac{x}{\Psi_Q(x)}, \hat{n}(x)\right\rangle
\end{align*}
from which we obtain
\begin{align*}
				\nabla \Psi_Q(x) = \frac{\Psi_Q(x)}{\left\langle x, \hat{n}(x)\right\rangle} \hat{n}(x).
\end{align*}
Since $\hat{n}$ is homogeneous of order zero, so is $\nabla \Psi_Q$. We have for $x \neq 0$
\begin{align*}
				\jac_{\nabla \Psi_Q(x)} = \jac \left( x \mapsto \nabla \Psi_Q\left( \frac{x}{\Psi_Q(x)} \right) \right) = \jac_{\nabla \Psi_Q(x/\Psi_Q(x))} \left( \frac{\Id_n}{\Psi_Q(x)} - \frac{x \nabla \Psi_Q(x)^T}{\Psi_Q(x)^2}\right).
\end{align*}
Now set $h(x) = \Psi_Q(x)^2/2$, which is convex and $C^1$ since 
\begin{align*}
				\nabla h(x) &=\Psi_Q(x)\nabla \Psi_Q(x) 
\end{align*}
is a continuous function. 
We have for $x \neq 0$
\begin{align*}
				\jac_{\nabla h(x)} &= \nabla \Psi_Q(x) \nabla \Psi_Q(x)^T + \Psi_Q(x) \jac_{\nabla \Psi_Q(x/\Psi_Q(x))} \left( \frac{\Id_n}{\Psi_Q(x)} - \frac{x \nabla \Psi_Q(x)^T}{\Psi_Q(x)^2}\right) \\
				&= \nabla \Psi_Q(x) \nabla \Psi_Q(x)^T + \jac_{\nabla \Psi_Q(x/\Psi_Q(x))} \left( \Id_n - \frac{x \nabla \Psi_Q(x)^T}{\Psi_Q(x)}\right).
\end{align*}
This expression remains bounded which shows that $\nabla h$ is Lipschitz and we may assume by rescaling that its Lipschitz constant is $1$. Set for all $x$, $x^+ = x - \nabla h(x)$, we have using standard arguments in the analysis of gradient descent on $h$, whose global minimum is the origin, that 
\begin{align*}
    2h(x^+) + \norm{x^+}{}^2 \leq \norm{x}{}^2.
\end{align*}
We have that $h$ is positively homogenoeous of degree $2$ so that
\begin{align*}
    h(x^+) = \norm{x^+}{}^2 \frac{h(x^+)}{\norm{x^+}{}^2} = \norm{x^+}{}^2 h\left( \frac{x^+}{\norm{x^+}{}}\right) \geq \norm{x^+}{}^2 \min_{\norm{y}{} = 1} h\left(y\right),
\end{align*}
where the minimum is attained and is positive, call it $c>0$. All in all, we have
\begin{align*}
    \norm{x - \nabla h(x)}{} \leq \frac{1}{\sqrt{1 + 2 c}} \norm{x}{},
\end{align*}
so that the constructed function complies with hypotheses of the Lemma, independently of $Q$.

%The strong convexity of $h$ locally around $0$ is equivalent to $\jac_{\nabla h}$ being positive definite locally around $0$ and this implies that $\cljac_{\nabla h}(0)$ contains only nonsingular matrices.

By definition of the gauge function, the sublevel sets of $\Psi_Q$ are of the form $\lambda Q$ for $\lambda\in\R$. If $h$ was strongly convex locally around $0$, one would have that its sublevel sets are also strongly convex (positively curved). This is not the case, for example if $Q$ is a square with smoothed corners. This shows that $h$ is not necessarily locally strongly convex. To ensure $h$ is semialgebraic, it suffices to take $Q$ a semialgebraic square with smoothed corners.

We conclude with the following implication: if $\cljac_{\nabla h}(0)$ contains only nonsingular elements then $h$ is strongly convex locally around $0$. Indeed in this case $\cljac_{\nabla h}(x)$ contains only positive definite elements for all $x$ in a convex compact neighborhood of $0$, set $\lambda > 0$ a lower bound on the minimum eigenvalue in this neighborhood (which exists by graph closedness and continuity of the smallest eigenvalue), we have by (Aumann) integration in \defref{def:conservativeJacobian} using conservativity of $\cljac_{\nabla h}$, for all $x,y$ in this neighborhood,
\begin{align*}
    \left\langle \nabla h(x) - \nabla h(y) , x - y \right\rangle  &= \left\langle \int_0^1 \cljac_{\nabla h}((1-t)x + ty) (x - y) dt , x - y \right\rangle\\
    &= \left\langle \int_0^1 \cljac_{\nabla h}((1-t)x + ty) dt (x - y) , x - y \right\rangle\\
    &\geq \lambda \norm{x - y}{}^2,
\end{align*}
which means strong monotonicity of $\nabla h$, equivalent to strong convexity of $h$. By contraposition, if $h$ is not locally strongly convex around $0$, then, $\cljac_{\nabla h}(0)$ contains singular matrices.
\end{proof}
\propref{prop:semialgebraic} shows that the equivalence between contractivity and \eqref{eq:contractionFB} does not hold in the $C^{1,1}$ case, higlighting a gap between $C^{1,1}$ and $C^2$ functions. This actually  extends to the nonsmooth setting with a proximal point step using convex analysis and Moreau enveloppes. 
\begin{proposition}[Linear convergence does not imply contractivity II]\label{prop:semialgebraicProx}
There exists a convex semialgebraic function $g\colon\R^2\to\R$ such that $g(0)=0$, $\prox_g$ is path differentiable, and, for some $0 < \rho < 1$, $\norm{\prox_g(x)}{}\leq \rho\norm{x}{}$ for all $x$, such that, for any convex conservative Jacobian $\consjac_{\prox_g}$, not every element of $\Id_2-\consjac_{\prox_g}(0)$ is invertible.
\end{proposition}
\begin{proof}
Let $h\in C^{1,1}(\R^2)$ be a function given by \propref{prop:semialgebraic}, it is convex and semialgebraic with $\nabla h$ $1$-Lipschitz, path differentiable, and $\|x-\nabla h(x)\|\leq \rho\|x\|$ for some $\rho\in(0,1)$ such that $h$ is not strongly convex locally around $0$ and $\cljac_{\nabla h}(0)$ contains singular elements. The function $\tilde{h} \colon x \mapsto \frac{\|x\|^2}{2} - h(x)$ is convex \cite[Theorem 18.15 (vi)]{bauschke2011convex} with $1$-Lipschitz gradient.  Recall that the \emph{conjugate} of $\tilde{h}$ is $\tilde{h}^*(x)\colon x \mapsto \sup\limits_{u}\ip{x,u}{}-\tilde{h}(u)$. The function  $g \colon x \mapsto \tilde{h}^*(x) - \frac{\|x\|^2}{2}$ is convex and semialgebraic, because $\tilde{h}^*$ is $1$-strongly convex and semialgebraic \cite[Proposition 10.8]{bauschke2011convex}. It satisfies $\prox_{g}(x) = \nabla \tilde{h}(x) = x - \nabla h(x)$ \cite[Corollary 24.5]{bauschke2011convex}, so that the gradient descent mapping for $h$ with unit step size is equivalent to the prox operator for $g$. 

Thus for all $x$, $\|\prox_g(x)\| = \|x-\nabla h(x)\|\leq \rho \|x\|$ for some $\rho\in(0,1)$, $\prox_g$ is path differentiable, and $\consjac_{\nabla h}: =\Id_2 - \consjac_{\prox_g}$ is a convex conservative Jacobian for $\nabla h$. Finally, by contraposition, $\Id_2 - \consjac_{\prox_g}(0) = \consjac_{\nabla h}(0)$ which contains at least one singular elemen by \propref{prop:semialgebraic} using the fact that $\cljac_{\nabla h}(0)\subset\consjac_{\nabla h}(0)$ by convexity of $\consjac_{\nabla h}(0)$.
\end{proof}

\begin{remark}[On local growth conditions]
The result of this section may be refined by considering local growth conditions, which are sufficient to ensure linear convergence of the forward-backward algorithm as in \eqref{eq:contractionFB}. Most important examples include global error bounds 
\begin{align*}
    g(x) - \min_z g(z) \geq \lambda \mathrm{dist}(x, \argmin g)^2
\end{align*}
for some $\lambda >0$ and for all $x$ (see \cite{pang1997error} 
and references therein) as well as global \kurdyka inequality 
\begin{align*}
    \min_{v\in \partial g(x)} \|v\| \geq \lambda' \sqrt{g(x) - \min_z g(z)} 
\end{align*}
for  some $\lambda' >0$ and for all $x$ 
(see \cite{bolte2017error} and references therein). In our setting (coercive, semialgebraic, unique critical value), these conditions are equivalent and are sufficient for \eqref{eq:contractionFB} to hold true \cite{bolte2017error}.   On the other hand, \eqref{eq:contractionFB} implies a quadratic \L ojasiewicz inequality (KL with exponent $1/2$) or a global quadratic error bound in our setting \cite[Proposition 4.19]{garrigos2017convergence}. Since all these conditions are equivalent the results of this section actually hold true replacing \eqref{eq:contractionFB} by a quadratic error bound or a quadratic  \L ojasiewicz inequality showing a fundamental limit to the extension of \thmref{thm:stronglyConvex} beyond strong convexity.

\end{remark}

%%%%%%%%%%%%%%%%%%%%%%%%%%%%%%%%%%%%%%%%%%%%%%%%%%%%%

\section{Applications to saddle point problems and duality}\label{sec:applications}

We demonstrate how to apply the previous sections' results to several different parametric optimization pro\-blems in which one seeks to differentiate the solution mapping as a function of the parameters $\theta$. In each of the following subsections, $\Theta\subset\R^p$ is a connected open set on which Assumption \ref{ass:functionPathDifferentiable} will be required to hold for various operators.

\subsection{Differentiating the dual solution of a convex composite problem}
Consider the following composite minimization problem
\nnewq{\label{eq:primal}
\min\limits_{x\in\R^n} f_\theta(x) + g_\theta(K_\theta x)
}
where, for each $\theta\in\Theta$, $f_\theta\in C^{1,1}(\R^n)$ is a strongly convex function, $g_\theta\colon\R^m\to\exR$ is a proper closed convex function, $K_\theta\colon\R^n\to\R^m$ is a surjective linear operator, and $\xstar(\theta)$ is the unique solution (the objective is proper by surjectivity).

The goal is to differentiate the solution $\xstar$ with respect to $\theta$, for which we assume that $\nabla f_\theta$ and $\prox_{\gamma g_\theta}$ are path differentiable, jointly in $(\theta,x)$, so that \assref{ass:functionPathDifferentiable} holds here. It is then possible to directly apply \thmref{thm:stronglyConvex} to differentiate $\xstar$ since \assref{ass:functionPathDifferentiable} holds and $f_\theta$ is strongly convex. However, because of the coupling between $g_\theta$ and $K_\theta$ in \eqref{eq:primal}, computing $\consjac_{\xstar}$ through this approach would necessitate computing $\prox_{g_\theta\circ K_\theta}$, which is nontrivial even when $\prox_{g_\theta}$ is known (unless $K_\theta$ is a (semi)orthogonal matrix \cite[Lemma 2.8]{combettes2005signal}).

We can instead use the generalized duality of Fenchel-Rockafellar, which will decouple the linear operator $K_\theta$ from $g_\theta$ in a way that is especially useful if $K_\theta$ is surjective, which we will assume. The  dual problem of \eqref{eq:primal} is given, for each $\theta\in\Theta$, by
\nnewq{\label{eq:dual}
-\min\limits_{y\in\R^m} f_\theta^*(-K_\theta^*y)+g_\theta^*(y)
}
to which we can apply our results, with $f_\theta^*(-K_\theta^*\cdot)$ and $g_\theta^*$ taking the role of $f$ and $g$ in the assumptions and theorems. Note that $y \mapsto f_\theta^*(-K_\theta^*y)$ is indeed strongly convex. To be explicit, we take $T(\theta,y) = \prox_{\gamma g_\theta^*}(y)$ and $S(\theta,y) = (\theta, y + \gamma K_\theta \nabla f_\theta^*(-K_\theta^* y))$ with $H(\theta,y) = T(S(\theta,y))$, so that the fixed point equation we are considering the dual solution $\ystar(\theta)$ to satisfy is
\nnewq{\label{eq:dualres}
\res(\theta,\ystar(\theta))=\ystar(\theta)-\prox_{\gamma g_\theta^*}(\ystar(\theta) + \gamma K_\theta\nabla f_\theta^*(-K_\theta^*\ystar(\theta))) =0.
}
\begin{theorem}[Path differentiability of the dual solution of a composite problem]
\label{thm:structuredDuality}
Consider \eqref{eq:dual} where, for each $\theta\in\Theta$,  $f_\theta\in C^{1,1}(\R^n)$ is $\alpha$-strongly convex with $\beta$-Lipschitz continuous gradient, $g_\theta\colon\R^m\to\exR$ is a closed convex proper function and $K_\theta\in\R^{m\times n}$ is surjective with singular values in $[\underline{\lambda}, \bar{\lambda}]$ for some $0<\underline{\lambda}\leq \bar{\lambda}$, uniformly in $\theta$. Then $y \mapsto f_\theta^*(-K_\theta^*y)$ is $\underline{\lambda}^2 / \beta$ strongly convex and has a gradient which is $\bar{\lambda}^2  / \alpha$-Lipschitz continuous, uniformly in $\theta$.

Assume furthermore that $\prox_{\gamma g_\theta^*}$, $\nabla f_{\theta}^*$ and $K_\theta$ are path differentiable so that \assref{ass:functionPathDifferentiable} holds with $f_\theta^*\circ[-K_\theta^*]$ and $g_\theta^*$.
Then, the unique dual solution $\ystar(\theta)$  of \eqref{eq:dual} is path differentiable on $\Theta$ with a conservative Jacobian given for all $\theta\in\Theta$ by
\newq{
\consjac_{\ystar}\colon \theta\rightrightarrows\brac{\para{\Id_n - V(\Id_n - \gamma Z)}^{-1}\para{U-\gamma VW} \colon \sbrac{U\ V}\in \cljac_1, \sbrac{W\ Z}\in\cljac_2}
}
where
\newq{
\cljac_1 := \cljac_{T}(\theta,\ystar(\theta)+\gamma K_\theta\nabla f_\theta^*(-K_\theta^*\ystar(\theta)))\quad\quad\mbox{and}\quad\quad
\cljac_2 := \cljac_{\nabla (f_\theta\circ[-K_\theta^*])}(\theta,\ystar(\theta)).
}
and $\gamma$ is any number in $ (0,\frac{2\underline{\lambda}^2/\beta}{(\underline{\lambda}^2/\beta + \bar{\lambda}^2/\alpha)})$. 
\end{theorem}
\begin{proof}
For each $\theta\in\Theta$, the function $g_\theta^*\colon\R^m\to\exR$ is a closed convex proper function since $g_\theta$ is, meanwhile the function $f_\theta^*\colon\R^n\to\R$ is $\frac{1}{\beta}$-strongly convex and differentiable with $\frac{1}{\alpha}$-Lipschitz continuous gradient \cite[Theorem 18.15(vii)]{bauschke2011convex}. The function $f_\theta^*(-K_\theta^*y)$ has gradient $-K_\theta\circ\nabla f_\theta^*\circ -K_\theta^*$ which is $\frac{\bar{\lambda}^2}{\alpha}$-Lipschitz continuous because $\nabla f_\theta$ is Lipschitz continuous with constant $\frac{1}{\alpha}$ and both $K_\theta$ and $K_\theta^*$ are $\bar{\lambda}$-Lipschitz continuous.
To demonstrate that $f_\theta^*\circ-K_\theta^*$ is $\frac{\underline{\lambda}^2}{\beta}$-strongly convex, we have for all $y$ 
\begin{align*}
    f_\theta^*(-K_\theta^* y) - \frac{\underline{\lambda}^2}{2\beta} \|y\|^2 =  \left(f_\theta^*(-K_\theta^* y) - \frac{1}{2\beta} \|K_\theta^* y\|^2\right) + \left(\frac{1}{2\beta} \|K_\theta^* y\|^2 -  \frac{\underline{\lambda}^2}{2\beta} \|y\|^2\right).
\end{align*}
The first term, $f_\theta^*(-K_\theta^* y) - \frac{1}{2\beta} \|K_\theta^* y\|^2$, is convex as it is the composition of a convex function $f(\cdot)-\frac{1}{2\beta}\|\cdot\|^2$ and a linear map $-K_\theta^*$. Indeed, by $\frac{1}{\beta}$-strong convexity of $f_\theta^*$, the function $f_\theta^*(\cdot)-\frac{1}{2\beta}\|\cdot\|^2$ is necessarily convex \cite[Proposition 10.6]{bauschke2011convex}. The second term, $\frac{1}{2\beta} \left(\|K_\theta^* y\|^2 -  \underline{\lambda}^2 \|y\|^2\right)$, is convex because the smallest eigenvalue of $K_\theta K_\theta^*$ is $\underline{\lambda}^2$. Hence the claimed strong convexity modulus of $\frac{\underline{\lambda}^2}{\beta}$, justifying the first part of the theorem claiming regularity of $f_\theta^*\circ -K_\theta^*$. Then, using the assumption that $\prox_{\gamma g_\theta}, \nabla f_\theta^*$, and $K_\theta$ are all path differentiable so that \assref{ass:functionPathDifferentiable} holds, we are finally able to apply \thmref{thm:stronglyConvex} to \eqref{eq:dual} and its fixed point formulation \eqref{eq:dualres} and the desired results follow.
\end{proof}
\begin{remark}[Path differentiability of the primal solution]\label{rem:dualSolutionToPrimal}
We can recover the primal solution from the dual solution through the equation $\xstar(\theta) = \nabla f_\theta^*(-K_\theta^*\ystar(\theta))$, coming from the primal-dual optimality conditions, since $\nabla f_\theta^*$ is path differentiable. Indeed, the functions $\prox_{\gamma g_\theta^*}$ and $\nabla f_\theta^*$ are path differentiable if $\prox_{\gamma g_\theta}$ and $\nabla f_\theta$ are assumed to be path differentiable. By the Moreau decomposition \cite[Theorem 14.3(ii)]{bauschke2011convex} we can express $\prox_{\gamma g_\theta^*}(y) = y - \prox_{g_\theta/\gamma}(y/\gamma)$. Meanwhile for $\nabla f_\theta^*$, we can invoke the path differentiable inverse function theorem \cite[Corollary 2]{bolte2021nonsmooth} with $\nabla f_\theta^* = (\nabla f_\theta)^{-1}$, the assumptions of which hold due to the fact that $\nabla f_\theta$ is path differentiable and $f_\theta^*$ is both Lipschitz-smooth and strongly convex. 
\end{remark}

\begin{example}[Learning sparsity priors]
The problem of learning a sparsity prior can be seen as a bilevel optimization problem \cite{ghanem2021supervised, peyre2011learning, saiprasad2015sparsifying} which fits the framework of this subsection. Given some set of training data $\brac{(u_1,\hat{u}_1),\ldots,(u_q,\hat{u}_q)}$ where $u_i$ is the ground truth for some signal (e.g., an image) and $\hat{u}_i$ is a noisy observation of $u_i$, we seek to find an optimal linear operator $K_\theta\in\R^{s\times n}$, the so-called sparsity prior. The general form of the problem can be cast as the following bilevel optimization problem,
\newq{
\min\limits_{\theta\in\R^p}\sum\limits_{i=1}^q\frac{1}{2}\norm{u_i - x_i(\theta)}{2}^2\quad\mbox{such that, }\forall i\in\brac{1,\ldots,q},\quad x_i(\theta)\in\argmin\limits_{x_i\in\R^n}\frac{1}{2}\norm{x_i-\hat{u}_i}{2}^2 + \norm{K_\theta x_i}{1}
}
where $\theta\in\R^{sn}$ with $K_\theta :=\begin{bmatrix} \theta_{1,1} & \ldots & \theta_{1,n}\\ \vdots & 
\ddots & \vdots\\ \theta_{s,1} & \ldots & \theta_{s,n}\end{bmatrix}$. Assume that $\Theta \subset \R^{sn}$ is a connected open set such that, for all $\theta\in\Theta$, $K_\theta$ is surjective and its singular values are contained in $[\underline{\lambda},\bar{\lambda}]$ for some $0<\underline{\lambda}\leq\bar{\lambda}$. Then the lower level problem 
matches exactly that of \eqref{eq:primal} with $f_\theta(x) = \sum\limits_{i=1}^q\frac{1}{2}\norm{x_i-
\hat{u}_i}{2}^2$ and $g_\theta(x)=\sum\limits_{i=1}^q\frac{1}{2}\norm{K_\theta x_i}{1}$, which we write 
here as sums even though they are separable in $x_i$. Note that $K_\theta$ is obviously not surjective for all $\theta \in \R^{sn}$ because of the general parametrization chosen. Instead of fixing the required open set $\Theta$, one could employ a different parameterization of $K_\theta$ with constraints on parameters ensuring that $K_\theta$ remains surjective. Regardless, by assuming surjectivity, Theorem \ref{thm:structuredDuality} applies and we can continue.

The dual of the inner problem, is given, for each 
$i\in\brac{1,\ldots,q}$, by
\newq{
y_i(\theta)\in\argmin\limits_{\{y_i\in\R^s: \norm{y_i}{\infty}\leq 1\}}\frac{1}{2}\norm{K_\theta^*y_i-\hat{u}_i}{2}^2
}
which has a fixed point equation 
\newq{
y_i^\star = \proj_{\mathcal{D}}(y_i^\star + \gamma K_\theta(K_\theta^*y_i^\star-\hat{u}_i)),
}
where $\proj_{\mathcal{D}}$ is the projection onto the $\ell^\infty$ unit ball in $\R^s$, i.e., the mapping whose coordinates are given by
$z\mapsto \sign(z)\min(1,\absv{z})$ component-wise. Using the notation of \secref{sec:monotone}, we 
have $T(\theta,y_i) = \proj_{\mathcal{D}}(y_i)$ and $S(\theta,y)=y_i^\star + \gamma 
K_\theta(K_\theta^*y_i^\star-\hat{u}_i)$. The primal solution $\xstar$ can be recovered from the dual 
solution through the relationship given in \remref{rem:dualSolutionToPrimal}, for each $\theta\in\Theta$ and  
$i\in\brac{1,\ldots,q}$,
\newq{
\xstar_i(\theta) = \nabla f_\theta^* (-K_\theta^*\ystar_i(\theta)) \implies \xstar_i(\theta) = \hat{u}_i-K_\theta^*\ystar_i(\theta).
}
We emphasize the difference in our approach to those taken in previous works \cite{ghanem2021supervised, peyre2011learning, saiprasad2015sparsifying}. While \cite{peyre2011learning} relies on a smoothing process for the $\ell^1$-norm in the lower level problem, \cite{saiprasad2015sparsifying} assumes that $K_\theta$ is an orthogonal matrix in contrast to our assumption that $K_\theta$ is surjective. In \cite{ghanem2021supervised}, the authors use unrolling on the algorithm used to solve the lower-level problem rather than implicit differentiation as we do.
\end{example}

\subsection{Differentiating the solutions of min-max problems}
Consider the following min-max problem
\nnewq{\label{eq:minmax}
\min\limits_{x\in X}\max\limits_{y\in Y} \Phi_\theta(x,y)
}
where $X \subset \R^n$ is closed and convex, and $Y \subset \R^m$ is convex compact and, for each $\theta\in\Theta$, $\Phi_\theta \colon \R^n \times \R^m \to \R$ is continuous such that $-\Phi_\theta(x,\cdot)\colon\R^m\to\R$ and $\Phi_\theta(\cdot,y)\colon\R^n\to\R$ are $\alpha$-strongly convex for each $x$ and for each $y$, respectively. Assume also that $\Phi_{(\cdot)}(x,y)$ is Lipschitz continuous on $\Theta$ for all $(x,y)\in\R^n\times\R^m$. The very general form of this problem encompasses min-max problems with nonlinear couplings of the form considered in \cite{hamedani2018primal, hamedani2021primal}. The solution mapping for this problem incorporates the primal and dual variables together, $\theta\mapsto (\xstar(\theta),\ystar(\theta))$. The optimality condition can be written for each $\theta\in\Theta$ as
\newq{
\begin{pmatrix}0 \\ 0\end{pmatrix} \in \begin{pmatrix}\partial_x \Phi_\theta +N_X& 0\\ 0 & -\partial_y \Phi_\theta +N_Y \end{pmatrix}\begin{pmatrix}\xstar(\theta)\\ \ystar(\theta)\end{pmatrix}
}
where $N_X$ and $N_Y$ denote respectively the normal cones to $X$ and $Y$. This is a special case of \eqref{eq:monotoneInclusion} with $\A_\theta =\begin{pmatrix}\partial_x \Phi_\theta +N_X & 0\\ 0 & -\partial_y \Phi_\theta +N_Y \end{pmatrix}$ and $\B_\theta\equiv 0$. Indeed from strong convexity, for each $\theta\in\Theta$, $\A_\theta$ is $\alpha$ strongly monotone, and from closedness of $X$ and compactness of $Y$ it can be shown that the range of $I + \mathcal{A}_\theta$ is $\R^n \times \R^m$ so that $\A_\theta$ is maximal \cite[Theorem 12.12]{rockafellar1998variational}.

For the function $\fb_\theta$ defined in \assref{ass:pathDifferentiable} applied to this problem, we have $\fb_\theta(x) = \mathcal{R}_{\A_\theta}(x)$ so that $\res(\theta,x) = x - \mathcal{R}_{\A_\theta}(x)$, leading to the following result.
\begin{theorem}[Path differentiability of min-max solutions]
\label{thm:minmax}
Consider \eqref{eq:minmax} where, for each $\theta\in\Theta$, $-\Phi_\theta(x,\cdot)\colon\R^m\to\exR$ and $\Phi_\theta(\cdot,y)\colon\R^n\to\exR$ are both closed $\alpha$-strongly convex proper functions. Assume that $\mathcal{R}_{\A_{\theta}}$ is path differentiable so that \assref{ass:pathDifferentiable} holds with the conservative Jacobians defined in \eqref{def:jacobianChoice} with $\gamma\in(0,1/\alpha)$. Then, the solution mapping $\theta\mapsto(\xstar(\theta),\ystar(\theta))$ is unique and path differentiable on $\Theta$ with a conservative Jacobian given for each $\theta\in\Theta$ by 
\newq{
\consjac_{(\xstar,\ystar)} \colon \theta\rightrightarrows\{\para{\Id_{(n+m)}-[V_1\ V_2]}^{-1}U \colon [U\ V_1\ V_2]\in\cljac_T(\theta, (\xstar(\theta), \ystar(\theta)))\}.
}
\end{theorem}
\begin{proof}
For each $\theta\in\Theta$, the maximal monotone operator $\A_\theta$ is $\alpha$-strongly monotone by the $\alpha$-strong convexity and $\alpha$-strong concavity of $\Phi_\theta(\cdot,y)$ and $\Phi_\theta(x,\cdot)$ respectively. Using this fact with \assref{ass:pathDifferentiable} and the conservative Jacobians defined in \eqref{def:jacobianChoice}, we can apply \thmref{thm:stronglyMonotone} with any $\beta\in(0,\infty)$ since $\B_\theta\equiv 0$. In particular, we can take $\beta=(\sqrt{2}-1)\alpha$ so that applying \thmref{thm:stronglyMonotone} requires $\gamma\in(0,1/\alpha)$. Therefore, the solution $(\xstar(\theta),\ystar(\theta))$ is path differentiable with a conservative Jacobian given by
\newq{
\consjac_{(\xstar,\ystar)} \colon \theta\rightrightarrows\{\para{\Id_{(n+m)}-[V_1\ V_2]}^{-1}U \colon [U\ V_1\ V_2]\in\cljac_T(\theta, (\xstar(\theta), \ystar(\theta)))\}.
}
\end{proof}

\subsection{Differentiating the solutions of primal-dual problems}
The min-max problem \eqref{eq:minmax} from the previous subsection is general in that it does not assume a particular coupling between $x$ and $y$. We turn now to primal-dual optimization with linear coupling, a well-known problem template that was studied, for instance, in \cite{chambolle2022accelerated, chambolle2011first, chambolle2016ergodic, condat2013primal, liang2018local, silveti2021stochastic} and allows to model many different problems coming from computer vision and machine learning. Consider the parametrized primal-dual problem
\begin{equation}\label{eq:primalDual}
\min\limits_{x\in\R^n} g_\theta(x) + \max\limits_{y\in\R^m}  \ip{K_\theta x,y}{} - f_\theta^*(y)
\end{equation}
where, for all $\theta\in\Theta$, $K_\theta\in\R^{m\times n}$ is a linear operator and both $f_\theta^*\colon\R^m\to\exR$ and $g_\theta\colon\R^n\to\exR$ are closed convex proper functions. In contrast to \cite{chambolle2021learning, bogensperger2022convergence}, we allow for the functions $g_\theta$ and $f^*_\theta$ to be parametrized by $\theta$, in addition to the linear operator $K_\theta$. For each $\theta\in\Theta$, the optimality conditions for a solution $(\xstar(\theta),\ystar(\theta))$ to this problem are
\begin{equation*}
K_\theta\xstar(\theta) \in \partial f_\theta^*(\ystar(\theta))\quad\mbox{and}\quad-K_\theta^*\ystar(\theta)\in\partial g_\theta(\xstar(\theta))
\end{equation*}
which can be equivalently written as
\begin{equation*}
\begin{pmatrix}0 \\ 0\end{pmatrix} \in \underbrace{\begin{pmatrix}\partial g_\theta & 0\\ 0 & \partial f_\theta^*\end{pmatrix}}_{\A_\theta}\begin{pmatrix}\xstar(\theta)\\\ystar(\theta)\end{pmatrix} + \underbrace{\begin{pmatrix}0 & K_\theta^*\\-K_\theta & 0\end{pmatrix}}_{\B_\theta}\begin{pmatrix}\xstar(\theta)\\\ystar(\theta)\end{pmatrix}
\end{equation*}
with $\A_\theta$ maximal monotone and $\B_\theta$ maximal monotone and $\beta$-Lipschitz for some $\beta = \norm{K_\theta^*}{\mathrm{op}}>0$. Despite the fact that the operator $\B_\theta$ is not cocoercive, we can still apply the results developed in \secref{sec:monotone} because it is $\beta$-Lipschitz.

\begin{theorem}[Path differentiability for primal-dual problems]\label{thm:structuredMinMax}
Consider \eqref{eq:primalDual} where, for each $\theta\in\Theta$, $g_\theta\colon\R^n\to\exR$ and $f^*_\theta\colon\R^m\to\exR$ are closed $\alpha$-strongly convex proper functions and $\beta\geq \norm{K_\theta}{\mathrm{op}}> 0$ for some $\beta>0$. Assume also, for each $\theta\in\Theta$, that $\prox_{g_\theta}$ and $\prox_{f_\theta}$ are path differentiable with conservative Jacobians chosen according to \eqref{def:jacobianChoice} so that \assref{ass:pathDifferentiable} holds with $\gamma\in(0,\frac{2\alpha}{(\alpha+\beta)^2})$. Then the mapping $\theta\mapsto (\xstar(\theta),\ystar(\theta))$ is unique and path differentiable on $\Theta$ with a conservative Jacobian given for each $\theta\in\Theta$ by
\newq{
\consjac_{(\xstar,\ystar)} \colon\theta\rightrightarrows\brac{-[V_1\ V_2]^{-1}U\colon U\in\consjac_{\theta,\res}(\theta,\xstar(\theta),\ystar(\theta)), [V_1\ V_2]\in\consjac_{(x,y),\res}(\theta,\xstar(\theta),\ystar(\theta))}
}
where
\newq{
\hspace{-12mm}&\consjac_{(x,y),\res}\colon (\theta,\xstar(\theta),\ystar(\theta))\rightrightarrows \\
&\begin{bmatrix}\Id_m - \cljac_{\prox_{\gamma g_\theta}} (\theta,\xstar(\theta) - \gamma K_\theta^*\ystar(\theta)) & -\cljac_{\prox_{\gamma g_\theta}}(\theta,\xstar(\theta)-\gamma K_\theta^*\ystar(\theta))\times(-\gamma K_\theta^*)\\ -\cljac_{\prox_{\gamma f^*_\theta}}(\theta,\ystar(\theta)+\gamma K_\theta\xstar(\theta))\times(\gamma K_\theta) & \Id_n - \cljac_{\prox_{\gamma f^*_\theta}}(\theta,\ystar(\theta) + \gamma K_\theta\xstar(\theta))\end{bmatrix}.
}
\end{theorem}
\begin{proof}
For each $\theta\in\Theta$, due to the assumed strong convexity of $g_\theta$ and $f_\theta^*$, the operator $\A_\theta = \begin{pmatrix} \partial g_\theta & 0\\ 0 & \partial f_\theta^*\end{pmatrix}$ is strongly monotone with some constant $\alpha>0$ and, since $K_\theta$ is a linear operator, $\B_\theta=\begin{pmatrix}0 & K_\theta^*\\-K_\theta & 0\end{pmatrix}$ is maximal monotone and Lipschitz with constant $\beta$. Combining this with \assref{ass:pathDifferentiable} and the conservative Jacobians chosen in \eqref{def:jacobianChoice}, the conditions to apply \thmref{thm:stronglyMonotone} with $\gamma\in(0,\frac{2\alpha}{(\alpha+\beta)^2})$ are met and the function $\theta\mapsto(\xstar(\theta),\ystar(\theta))$ is path differentiable. More explicitly, we have for all $\theta\in\Theta$, $(x,y)\in\R^n\times\R^m$,
\newq{
\fb(\theta,x,y) = \begin{pmatrix}\prox_{\gamma g_\theta}(x-\gamma K_\theta^*y) \\ \prox_{\gamma f_\theta^*}(y+\gamma K_\theta x)\end{pmatrix}
}
 so that 
\newq{
\res(\theta,\xstar(\theta),\ystar(\theta))=\begin{pmatrix}\xstar(\theta)\\\ystar(\theta)\end{pmatrix} - \begin{pmatrix}\prox_{\gamma g_\theta}(\xstar(\theta)-\gamma K_\theta^*\ystar(\theta)) \\ \prox_{\gamma f_\theta^*}(\ystar(\theta)+\gamma K_\theta \xstar(\theta))\end{pmatrix}.
}
Using  \eqref{def:jacobianChoice}, the resulting conservative Jacobian for $(\xstar(\theta),\ystar(\theta))$ on $\Theta$ is
\newq{
&\consjac_{(\xstar,\ystar)} \colon \theta\rightrightarrows\\
&\{\para{\Id_{n+m}-[V_1\ V_2](\Id_{n+m}-\gamma [Z_1\ Z_2])}^{-1}\para{U-\gamma [V_1\ V_2]W} \colon [U\ V_1\ V_2]\in\cljac_1, [W\ Z_1\ Z_2]\in \cljac_2\}.
}
where 
\newq{
\cljac_1&=\cljac_T(\theta, \xstar(\theta)-\gamma K_\theta^*\ystar(\theta), \ystar(\theta)+\gamma K_\theta \xstar(\theta))\\
} 
and 
\newq{
\cljac_2=\cljac_{\B}(\theta, \xstar(\theta), \ystar(\theta)).
}
Alternatively, we can write
\newq{
\consjac_{(\xstar,\ystar)} \colon \theta\rightrightarrows\{\sbrac{V_1\ V_2}^{-1}U \colon U \in \consjac_{\theta, \res}(\theta,\xstar(\theta),\ystar(\theta)), [V_1\ V_2]\in\consjac_{(x,y),\res}(\theta,\xstar(\theta),\ystar(\theta))\}
}
where
\newq{
\hspace{-12mm}&\consjac_{(x,y),\res}\colon (\theta,\xstar(\theta),\ystar(\theta))\rightrightarrows \\
&\begin{bmatrix}\Id_m - \cljac_{\prox_{\gamma g_\theta}} (\theta,\xstar(\theta) - \gamma K_\theta^*\ystar(\theta)) & -\cljac_{\prox_{\gamma g_\theta}}(\theta,\xstar(\theta)-\gamma K_\theta^*\ystar(\theta))\times(-\gamma K_\theta^*)\\ -\cljac_{\prox_{\gamma f^*_\theta}}(\theta,\ystar(\theta)+\gamma K_\theta\xstar(\theta))\times(\gamma K_\theta) & \Id_n - \cljac_{\prox_{\gamma f^*_\theta}}(\theta,\ystar(\theta) + \gamma K_\theta\xstar(\theta))\end{bmatrix}.
}
\end{proof}
As in \cite{bogensperger2022convergence} and \cite{chambolle2021learning}, we have assumed that both $g_\theta$ and $f_\theta^*$ are strongly convex; in contrast to \cite{bogensperger2022convergence} we do not assume that $\prox_{g_\theta}$ or $\prox_{f_\theta^*}$ are differentiable at any specific points, nor do we assume that $g_\theta$ or $f_\theta$ are twice differentiable as in \cite{chambolle2021learning}.

In \cite{chambolle2021learning} the authors consider using a bilevel optimization problem to learn the best discretization of the total variation for inverse problems in imaging. The proposed bilevel problem was further studied in \cite{bogensperger2022convergence} where the authors noted the theoretical difficulties in differentiating the solution to a nonsmooth optimization problem with respect to some parameters. Indeed, they have no guarantees that their algorithm will avoid the set of points where the prox operator is not differentiable, nor can they show that the solutions will be points of differentiability. On the other hand, so long as the prox operator is path differentiable, our results apply despite these obstacles, and one can compute implicit conservative gradients.

\subsection{Automatic differentiation of algorithms}

The results of \thmref{thm:stronglyMonotone} (in particular, the fact that $\res$ is contractive) imply that the forward-backward splitting algorithm applied to solve monotone inclusions of the form in the theorem will converge linearly for $\gamma\in\left(0,\frac{2\alpha^2}{(\alpha+\beta)^2}\right)$. We illustrate our results with implicit differentiation which only requires the solution of the inclusion problem and does not depend on the algorithm used to solve it. Yet, it is worth emphasizing that under \assref{ass:pathDifferentiable}, the contractivity property in \defref{cond:contractivity} is sufficient to apply the convergence result of \cite{bolte2022automatic} to the forward-backward algorithm in our context. More precisely, Definition \ref{cond:contractivity} is precisely the same as \cite[Assumption 1]{bolte2022automatic} applied to the forward-backward algorithm to solve \eqref{eq:introProblem}. Therefore, in addition to path differentiability of the solution map, assumptions of Theorem \ref{thm:stronglyMonotone} provide a sufficient condition to ensure that automatic differentiation of the forward-backward algorithm generates a sequence of conservative jacobians such that conservativity is preserved asymptotically: the limits of iterative differentiation conservative jacobians form a conservative jacobian for the solution of \eqref{eq:introProblem} \cite[Corollary 1 and 2]{bolte2022automatic}. 

As a consequence, \thmref{thm:stronglyMonotone} implies that the iterative derivative convergence results of \cite{bolte2022automatic} apply to the forward-backward algorithm in all special cases described in the paper: \thmref{thm:stronglyConvex} for strongly convex optimization problems, \thmref{thm:structuredDuality} for solutions to primal problems obtained from their  dual, \thmref{thm:minmax} for general strongly monotone saddle point problems and \thmref{thm:structuredMinMax} for strongly monotone structured saddle point problems.

%%%%%%%%%%%%%%%%%%%%%%%%%%%%%%%%%%%%%%%%%%%%%%%%%%%

\section{Conclusion}\label{sec:conclusion}
We have presented sufficient conditions in the form of strong monotonicity, under which the path differentiability of the solution to a monotone inclusion problem is satisfied. As special cases, we have derived conditions that ensure path differentiability of solutions to a large class of nonsmooth parametric convex optimization problems - those which can be written as the sum of two parametric convex functions, one smooth and one possibly nonsmooth. By expressing the monotone inclusions as equivalent fixed point equations using the resolvent mapping, we were able to leverage path differentiability and the recently developed nonsmooth implicit path differentiation theorem of \cite[Corollary 1]{bolte2021nonsmooth} to deduce regularity of $\xstar$. Most importantly, we were able to characterize and give a formula for a conservative Jacobian of $\xstar$ with respect to $\theta$ using only the Clarke Jacobians associated with the resolvent mapping $\mathcal{R}_{\gamma \A_\theta}$ and the operator $\B_\theta$.

While this work is primarily theoretical, our results also lend insight to practical applications, e.g., the design of implicit neural network layers defined using convex optimization problems. Ensuring that an implicit layer is compatible with training is an important part of implicit layer design and, consequently, ensuring the invertibility condition is a necessary part of guaranteeing that training will work (c.f., \cite[Section 5]{bolte2021nonsmooth}). Besides this, our results highlight the relevancy of the typical strong convexity assumption made on the lower-level problem of a bilevel optimization problem to ensure implicit conservative Jacobians will exist when this lower-level problem is nonsmooth.

It is important also to understand the dependence of $\xstar$ on the step size $\gamma$ taken in the definition of $\fb$. In the smooth case, where $\A_\theta\equiv 0$ and $\B_\theta = \nabla f_\theta$ for some twice differentiable convex function $f_\theta$, there is no dependence on $\gamma$ which is to be expected since $0=\B_\theta(\xstar)$ can be differentiated directly without introducing $\gamma$. We delay exploring such questions for future work.

At a few points in the paper, we encounter objects which are not canonical or which can be chosen in multiple ways. We collect here these instances and elaborate on why we've chosen the way we have and the possible consequences of choosing to formulate things differently.

\paragraph{Choice of $\fb$} Our selection for $\fb$ in $\res$ can be attributed to the additive structure of $\A_\theta+\B_\theta$ and the Lipschitz continuity of $\B_\theta$. There are alternatives, for instance, considering the resolvent directly $\mathcal{R}_{\A_\theta+\B_\theta}$ or considering Douglas-Rachford, Peaceman-Rachford, etc, and extensions of this flavor are a matter for future research.

\paragraph{Choice of $\consjac_\fb$} We use multiplication of Clarke Jacobians in the formula for $\consjac_{\fb}$ given in \eqref{def:jacobianChoice} because it allows to obtain sufficient condition on problem data in \eqref{eq:monotoneInclusion} to ensure the contractivity of the residual equation. On the other hand, all the corollaries  of the paper would hold similarly with arbitrary conservative Jacobians for $\mathcal{R}_{\gamma \A}$ and $\B$ combined in a similar way, provided that this is compatible with the contractivity in \defref{cond:contractivity}. However, in this case, the contractivity defined in \defref{cond:contractivity} has to be explicitly assumed, not deduced from properties of problem data, because conservative Jacobians can be changed on a set of measure zero.

\paragraph{Choice of $\Theta$} The set $\Theta$ could be the whole space $\R^p$ or possibly a subset, for example if one of the operators $\A_\theta$ or $\B_\theta$ is not defined for every $\theta\in\R^p$, or if some of conditions (Lipschitz continuity, strong monotonicity, etc) can only be ensured to hold on some subset. The set $\Theta$ can also be taken as a neighborhood of some point $\bar{\theta}\in\R^p$ of interest, in which case it's possible to obtain local versions of all of our later results. For these local versions, we need only to assume that the contractivity in \defref{cond:contractivity} holds at the single point $\bar{\theta}$, from which it follows that there is some open neighborhood $\Theta$ on which the inequality must hold.

%%%%%%%%%%%%%%%%%%%%%%%%%%%%%%%%%%%

\section*{Acknowledgements}
The authors acknowledge the support of the Air Force Office of Scientific Research, Air Force Material Command, USAF, under grant numbers FA9550-19-1-7026. JB and EP acknowledge the support of ANR-3IA Artificial and Natural Intelligence Toulouse Institute, and thank Air Force Office of Scientific Research, Air Force Material Command, USAF, under grant numbers  FA8655-22-1-7012 ANR MasDol. JB  acknowledges the support of ANR Chess, grant ANR-17-EURE-0010.

\appendix
\small
\section{Appendix}
The next lemma is a general result that gives a bound on the Lipschitz constant of the forward mapping $\Id-\gamma \B$ in terms of the Lipschitz constant $\beta$ of the maximal monotone operator $\B$. This bound is slightly better than the more obvious bound $1+\gamma \beta$, and this improvement is crucial to prove \thmref{thm:stronglyMonotone}.
\begin{lemma}\label{lem:lipschitzConstant}
Let $\B\colon\R^n\to\R^n$ be a $\beta$-Lipschitz continuous maximal monotone operator, then the map $\Id-\gamma\B_\theta$ is $\sqrt{1+\gamma^2\beta^2}$-Lipschitz continuous.
\end{lemma}
\begin{proof}
For all $x,y\in\R^n$,
\newq{
\norm{(\Id-\gamma\B)(x) - (\Id-\gamma\B)(y)}{}^2 &= \norm{x-y}{}^2 -2\gamma\ip{x-y, \B(x) - \B(y)}{} + \norm{\gamma (\B(x) - \B(y))}{}^2\\
&\leq \norm{x-y}{}^2 + \gamma^2\norm{\B(x) - \B(y)}{}^2\\
&\leq (1+\gamma^2\beta^2)\norm{x-y}{}^2
}
where we have used the monotonicity of $\B$ followed by the $\beta$-Lipschitz continuity of $\B$ for the first and second inequalities, respectively. Thus, taking square roots, the mapping $\Id-\gamma\B$ is $\sqrt{1+\gamma^2\beta^2}$-Lipschitz continuous.
\end{proof}
%%%%%%%%%%%%%%%%%%%%%%%%%%%%%%%%%%%
\begin{small}
\bibliographystyle{plain}
\bibliography{references}
\end{small}
%%%%%%%%%%%%%%%%%%%%%%%%%%%%%%%%%%%
\end{document}